\documentclass[a4paper]{llncs}

\usepackage{times}
\usepackage{helvet}
\usepackage{courier}
\usepackage{color}
\usepackage{amsmath}
\usepackage{amssymb}
\usepackage{amsbsy}
\usepackage{graphicx}
\usepackage[ruled,vlined,linesnumbered]{algorithm2e}
\usepackage{subcaption}
\usepackage{xspace}
\usepackage{graphicx}
\usepackage{adjustbox}

\usepackage{mathtools}

\usepackage{url}
\urldef{\mailsa}\path|shufeng.kong@student.uts.edu.au, sanjiang.li@uts.edu.au|
\urldef{\mailsb}\path|michael.sioutis@oru.se|    
\newcommand{\keywords}[1]{\par\addvspace\baselineskip
\noindent\keywordname\enspace\ignorespaces#1}

\usepackage{mathtools}
\usepackage{caption}
\usepackage{float}
\usepackage{textcomp}
\usepackage{xspace}
\usepackage{graphicx}
\usepackage{adjustbox}
\usepackage{blindtext}
\usepackage[ruled,vlined,linesnumbered]{algorithm2e}
\usepackage{algorithmic}
\usepackage{xspace}

\usepackage[hidelinks,draft=false]{hyperref}
\AtBeginDocument{}
\setkeys{Gin}{draft=false}      
\usepackage[obeyDraft,color=white,colorinlistoftodos]{todonotes}

\colorlet{JL}{green!20!yellow!60!white} \colorlet{JLline}{JL!80!black}

\colorlet{SL}{red!20!yellow!60!white}
\colorlet{SLline}{SL!80!black}

\colorlet{SK}{blue!10!white}
\colorlet{SKline}{SK!80!black}

\SetFuncSty{textsc}
\SetProcNameSty{textsc}
\SetKw{KwSt}{s.t.}
\SetKw{KwBreak}{break}

\newcommand{\la}{\langle}
\newcommand{\ra}{\rangle}
\newcommand{\D}{\mathcal{D}}

\newcommand{\N}{{\cal N}}

\newcommand{\black}[1]{\textcolor{black}{#1}}

\newcommand{\dpc}{\ensuremath{\mathbf{DPC}}\xspace}
\newcommand{\dpcplus}{\ensuremath{\mathbf{DPC^*}}\xspace}

\newcommand{\csp}{\mathbf{CSP}}

\usepackage{wrapfig}

\usepackage{pgfplots}

\DeclarePairedDelimiterX\Set[1]\{\}{%
  
  #1 }

\begin{document}


\mainmatter  

\title{Exploring Directional Path-Consistency for Solving {Constraint Networks}}

\titlerunning{Lecture Notes in Computer Science: Authors' Instructions}

%
%
\author{Shufeng Kong\textsuperscript{1} \and Sanjiang Li\textsuperscript{1} \and Michael Sioutis\textsuperscript{2}}
%
%
\institute{\textsuperscript{1} \text{ } QSI, FEIT, University of Technology Sydney, Sydney, Australia\\
\mailsa\\
\textsuperscript{2} \text{ } AASS, \"{O}rebro University, \"{O}rebro, Sweden\\
\mailsb\\
}

%
%

\toctitle{Lecture Notes in Computer Science}
\tocauthor{Authors' Instructions}
\maketitle

\begin{abstract}
Among the local consistency techniques used for solving constraint networks, path-consistency (PC) has received a great deal of attention. However, enforcing PC is computationally expensive and sometimes even unnecessary. Directional path-consistency (DPC) is a weaker notion of PC that considers a given variable ordering and can thus be enforced more efficiently than PC. {This paper shows that \dpc (the DPC enforcing algorithm of Dechter and Pearl) decides the constraint satisfaction problem (CSP) of a constraint language $\Gamma$ if it is complete and has {the} variable elimination property (VEP). However, we also show that no complete VEP constraint language {can} have a domain with more than 2 values.}

\quad\quad {We then present a simple variant of the \dpc algorithm, called \dpcplus, and show that the CSP of a constraint language can be decided by \dpcplus if it is  closed under a majority operation. In fact, \dpcplus is sufficient for guaranteeing backtrack-free search for such constraint networks. Examples of majority-closed constraint classes include the classes of connected row-convex (CRC) constraints and tree-preserving constraints, which have found applications in various domains, such as scene labeling, temporal reasoning, geometric reasoning, and logical filtering. Our experimental evaluations show that \dpcplus significantly outperforms the state-of-the-art algorithms for solving majority-closed constraints.} 


\end{abstract}

\section{Introduction}

Many Artificial Intelligence tasks can be formulated as \emph{constraint networks} \cite{montanari1974networks}, such as natural language parsing \cite{maruyama1990structural}, temporal reasoning \cite{DechterMP91,planken2008p3c}{,} and spatial reasoning \cite{LiLW13}. A constraint network {comprises} a set of variables {ranging over some domain of possible values, and a set of constraints that specify allowed value combinations for these variables}. Solving a constraint network amounts to {assigning} values {to its variables} {such that} its constraints are satisfied. \emph{{Backtracking search}} is the principal mechanism for solving a constraint network{;} it assigns values to variables in a depth-first manner, and {backtracks} to the previous {variable} assignment {if there are no consistent values for the variable at hand}. \emph{Local consistency} techniques are commonly used to reduce the size of the search space before commencing search. However, searching for a complete solution for a constraint network is still hard{. In fact, even deciding whether the constraint network has a solution is NP-complete in general}. Therefore, {given a particular form of local consistency,} {a crucial} {task is to determine problems that can be solved by \emph{backtrack-free} search using that local consistency \cite{Freuder82}.} 

This paper considers a particular form of local consistency, called \emph{path-consistency} (PC), which is one of the most important and heavily studied  local consistencies in the literature (see e.g. \cite{mackworth1977consistency,mohr1986arc,DBLP:journals/ijait/Singh96,bliek1999path,chmeiss1998efficient}). Recently, it was {shown} that {PC can be used to decide {the satisfiability of a problem} if and only if the problem does not have the \emph{ability to count}} \cite{barto2014constraint,barto2014collapse}; however, it remains unclear whether backtrack-free search can be used to extract a solution for such a problem after enforcing PC.

Directional path-consistency (DPC) \cite{dechter1987network} is a weaker {notion} of PC that {considers a given variable ordering and} can {thus} be enforced more efficiently {than PC}. 
The DPC enforcing  algorithm of Dechter and Pearl \cite{dechter1987network}, denoted by \dpc, has been used to efficiently solve reasoning problems {in} temporal reasoning \cite{DechterMP91,planken2008p3c} and spatial reasoning \cite{sioutis2016efficiently}. It is then natural to ask \emph{what binary constraint networks {with finite domains} can be solved by \dpc}. Dechter and Pearl \cite{dechter1987network} \black {showed} that \dpc is sufficient {for enabling} backtrack-free search for networks with constraint graphs of \emph{regular width} 2. We consider the aforementioned question  
in the context of \emph{constraint languages}, which is a widely adopted approach in the study of tractability of constraint satisfaction problems \cite{CarbonnelC15}.  {Specifically}, we are interested in finding all binary constraint languages $\Gamma$ {such that} the consistency of any constraint network defined over $\Gamma$ can be decided by  {\dpc}.

To this end, we first exploit the close connection between DPC and \emph{variable elimination} by defining constraint languages that have the (weak) variable elimination property (VEP) (which will become clear in Definition~\ref{dfn:vep}). We call a constraint language $\Gamma$  \emph{complete} if it contains all relations that are definable in $\Gamma$ {(in the sense of {Definition}~\ref{dfn:definable})}. Then, we show that the constraint satisfaction problem (CSP) of a complete constraint language $\Gamma$ can be decided by \dpc if it is complete and has VEP, which is shown to be equivalent to the Helly property. However, we also show that no complete VEP constraint language {can} have a domain with more than 2 values. 

We then present a simple variant of the algorithm \dpc, called \dpcplus, and show that the consistency of a constraint network can be decided by \dpcplus if it is defined over any majority-closed constraint language.  In fact, we show that \dpcplus is sufficient {for guaranteeing} backtrack-free search for {such} constraint networks. Several important constraint classes have been found to be majority-closed. The most well-known one is the class of \emph{connected row-convex (CRC)} constraints \cite{DevilleBH99}, which is further generalized to a larger class of \emph{tree-preserving} constraints \cite{KongLLL15}. The class of CRC constraints has been successfully applied to temporal reasoning \cite{DBLP:conf/aips/Kumar05}, logical filtering \cite{DBLP:conf/aips/KumarR06}, and geometric reasoning \cite{kumar_geometric_2004}, and the class of tree-preserving constraints can model a large subclass of the scene labeling problem \cite{KongLLL15}. We also conduct experimental evaluations to compare \dpcplus to the state-of-the-art algorithms for solving majority-closed constraints, and show that \dpcplus significantly outperforms the latter algorithms.


The remainder of this paper is organized as follows. In {Section~\ref{sec:preliminaries}} we introduce basic notions and results that will be used throughout the paper. In Section~\ref{sec:DPC_Algo} we present the \dpc algorithm, and {in Section~\ref{sec:DPC_VE} we discuss} the connection between \dpc and variable elimination. 
In Section~\ref{sec:majority-closed} we prove that a complete constraint language $\Gamma$ has weak VEP if and only if $\Gamma$ is majority-closed. We then present in Section~\ref{dpcplus} our variable elimination algorithm \dpcplus, and empirically evaluate \dpcplus in Section~\ref{evaluation}. Finally, Section~\ref{conclusion} concludes the paper.

\section{Preliminaries}\label{sec:preliminaries}
This section recalls necessary notions and results.
\begin{definition} 
A \emph{binary constraint network} (BCN) $\mathcal{N}$ is a triple $\la V,{\D},C \ra$, where
\begin{itemize}
\item $V = \{v_1,\ldots,v_n\}$ is a nonempy finite set of variables;
\item ${\D}=\{D_1,\ldots,D_n\}$, where $D_i$ is the domain of $v_i$;
\item $C= \{(s_p, R_p)\mid 1\leq p\leq m\}$ is a set of \emph{binary constraints}, where $s_p=(v_i,v_j)(i\ne j)$ (called the \emph{scope} of $(s_p,R_p)$) is a pair of variables in $V$  and $R_p$ (called the \emph{constraint relation} of $(s_p,R_p)$)  is a subset of $D_i \times D_j$. 
\end{itemize} 
\end{definition}

Given a BCN $\mathcal{N}=\la V,{\D},C\ra$ and any pair of variables $(v_i,v_j)$ with $v_i,v_j\in V$ and $v_i\not=v_j$, we assume that there exists at most one constraint between {the pair}. For simplicity, we will {often} denote the {constraint} between $v_i$ and $v_j$ by $R_{ij}$. Further, we assume $R_{ij}=R_{ji}^{-1}$, which is the inverse of $R_{ji}$. We write $R_{ij} \in C$ throughout the paper to {state} that a constraint with scope $(v_i,v_j)$ is in $C$. Usual operations on relations such as intersection ($\cap$), composition ($\circ$), and inverse~($^{-1}$) are also assumed.

A \emph{partial solution} of $\mathcal{N}$ w.r.t. a subset $V^\prime$ of $V$ is an assignment of values to variables in $V^\prime$ such that all of the constraints $R_{ij}$ with $v_i,v_j\in V'$  are satisfied. A partial solution w.r.t. $V$ is called a \emph{solution} of $\mathcal{N}$. We say that $\mathcal{N}$ is \emph{consistent} or \emph{satisfiable} if it admits \black {a} solution, and \emph{inconsistent} or \emph{unsatisfiable} otherwise. A BCN $\mathcal{N}=\la V,{\D},C\ra$ is said to be \emph{globally consistent} if every partial solution w.r.t. $V^\prime\subseteq{V}$ can be consistently extended to a solution w.r.t. $V$. Further, $\mathcal{N}$ is said to be \emph{trivially inconsistent} if $C$ contains an empty constraint or $\D$ contains an empty domain. \black {Two BCNs 
are \emph{equivalent} if they have the same set  of solutions.}

\begin{example}\label{example-1}
Consider a BCN $\mathcal{N}=\la V, \D, C\ra$, where 
\begin{itemize}
\item $V=\Set{v_1, v_2, v_3, v_4}$;
\item $D_i = \Set{a,b,c}$ for $i=1,2,3,4$;
\item $C=\Set{R_{12},R_{23},R_{34},R_{24}}$ where

{\small
$R_{12}=
\begin{bmatrix}
1 & 1 & 1\\
1 & 1 & 0\\
1 & 0 & 0
\end{bmatrix}
$,   
$R_{34}=
\begin{bmatrix}
1 & 1 & 1\\
1 & 0 & 0\\
1 & 0 & 0
\end{bmatrix}
$, 
and $R_{32}=R_{24}=
\begin{bmatrix}
1 & 1 & 1\\
0 & 1 & 1\\
0 & 0 & 1
\end{bmatrix}
$.
}
\end{itemize}
We use Boolean matrices to represent binary relations. For example, $R_{12}$ represents the relation $\Set{\la a,a\ra, \la a,b\ra, \la a,c\ra, \la b,a\ra, \la b,b\ra, \la c,a\ra}$ between $v_1$ and $v_2$, where the values in both $D_1$ and $D_2$ are ordered as $a\prec b \prec c$. It is easy to check that $\sigma=\la a,a,a,a\ra$ is a solution of $\mathcal{N}$. 
\end{example}
\begin{figure}[htb]
    \centering   \includegraphics[width=0.3\linewidth]{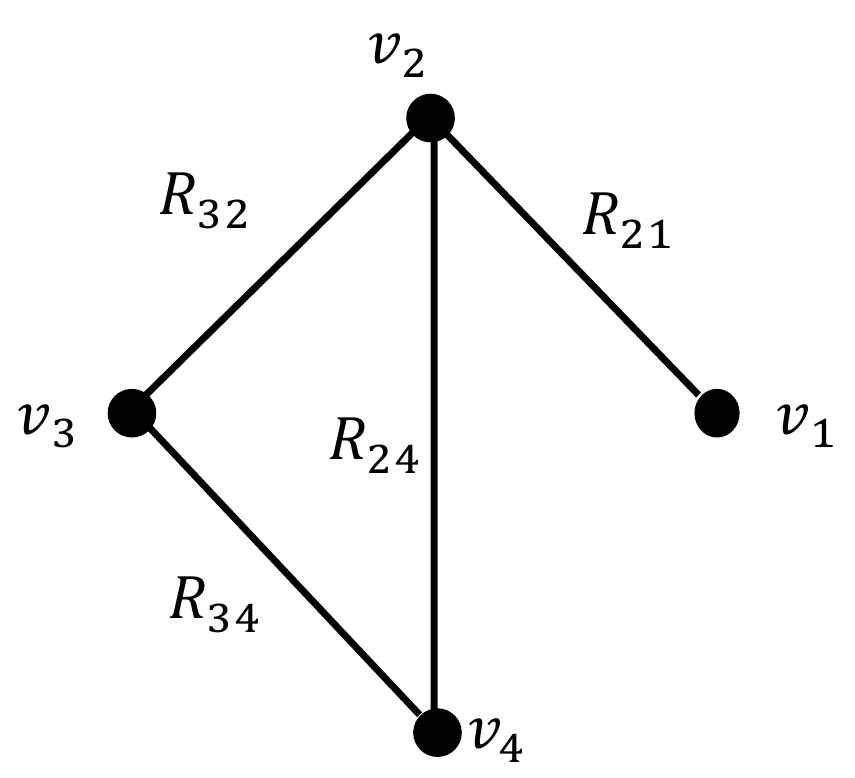}
    \caption{The constraint graph $G_{\N}$ of the BCN in Example~\ref{example-1}. Note that $\la v_1,v_2,v_3,v_4\ra$ is a PEO in $G_{\N}$, whereas $\la v_2,v_1,v_3,v_4\ra$ is not. \label{fig:exam1}}
  \end{figure}
  
The \emph{constraint graph} {$G_{\mathcal{N}}$} of $\mathcal{N}=\la V,{\D},C\ra$ is the \emph{undirected} graph $(V,E)$, where $e_{ij}\in E$ iff $R_{ij}\in C$. We assume $e_{ij}$ is always labeled with its corresponding constraint $R_{ij}$. Fig.~\ref{fig:exam1} shows the constraint graph of the BCN in Example~\ref{example-1}.


\begin{figure}[t]
    \centering   \includegraphics[width=0.4\linewidth]{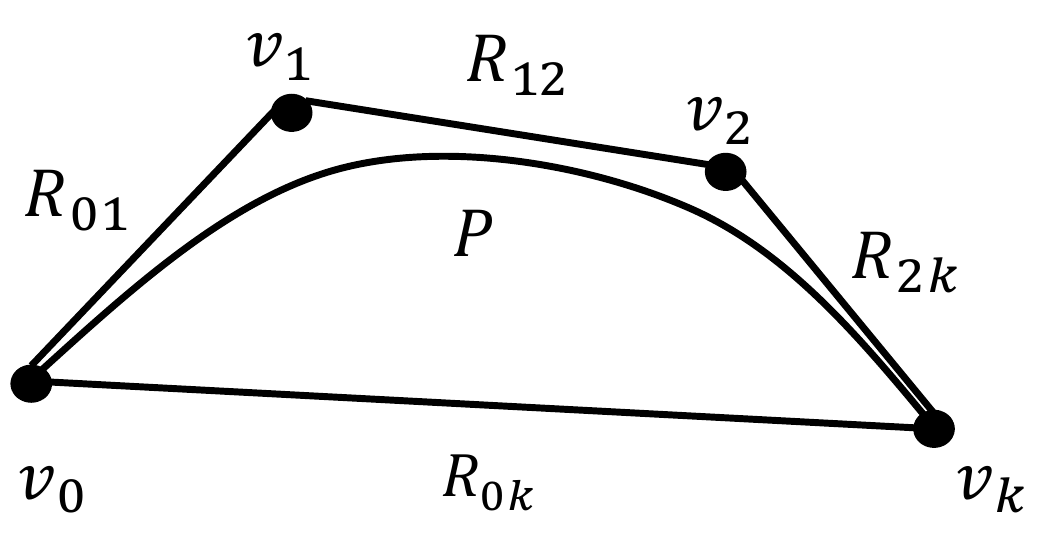}
    \caption{Path-Consistency \cite{bliek1999path}. \label{PC_example}}
  \end{figure}


An undirected graph $G$ is \emph{triangulated} or \emph{chordal} if every cycle of length greater than 3 has a \black {\emph{chord}}, \emph{i.e.}, an edge connecting two non-consecutive vertices of the cycle. The constraint graph $G_{\mathcal{N}}$ of a network $\mathcal{N}=\la V,{\D},C \ra$ can be completed or triangulated by adding new edges $e_{ij}$ {labeled with the universal constraint {$D_i\times D_j$}.}

Triangulated graphs play a key role in efficiently solving large sparse constraint networks
\cite{bliek1999path,planken2008p3c,sioutis2016efficiently}. A graph $G$ is triangulated iff it admits a \emph{perfect elimination ordering} (PEO) \cite{fulkerson1965incidence}. An ordering
  $\mathord{\prec}$ of the vertices of a graph $G = (V, E)$ is a PEO if $F_v = \{w \mid (v,w) \in E, v \prec w\}$ for all of $v \in V$, \black {\emph{i.e.,}} the set of successors of  $v$ \black {in the ordering} induces a complete subgraph of $G$ (see Fig.~\ref{fig:exam1} for an example).

A variable $v_i$ is \emph{arc-consistent} (AC) relative to a variable $v_j$ (or $R_{ij}$) if for any $a\in D_i$ {we} have some $b\in D_j$ such that $\langle a,b \rangle \in R_{ij}$. Given a constraint graph $G_{\mathcal{N}}=(V,E)$, {an edge $e_{ij}$ of $G_{\mathcal{N}}$ is AC} if $v_i$ is AC relative to $v_j$. 
Let $\pi=(v_0,\cdots,v_i,\cdots,v_k)$ be a path in $G_{\mathcal{N}}$ with $e_{0k} \in E$. We say {that} $\pi$ is \emph{path-consistent} (PC) if for all $\langle c_0,c_k \rangle \in R_{0k}$ {we} can find values for all intermediate variables $v_i$ $(0<i<k)$ such that all the constraints $R_{i,i+1}\ (0\leq i<k)$ are satisfied. See Fig.~\ref{PC_example} for an illustration. In particular, {$(v_0,v_2)$ is PC} relative to a third vertex $v_1$ if $e_{01},e_{12}$, and $e_{02}$ are all in $E$ and the path $\pi=(v_0,v_1,v_2)$ is PC. A constraint graph $G_{\mathcal{N}}$ is AC (resp. PC) iff all edges (resp. paths) in $G_{\mathcal{N}}$ are AC (resp. PC); $G_{\mathcal{N}}$ is \emph{strong\black {ly} PC} iff it is both AC and PC. A  constraint network $\mathcal{N}$ is AC if $G_{\mathcal{N}}$ is AC and $\N$ is PC if the competition of $G_{\N}$ is PC \cite{bliek1999path}.

Consider the BCN $\mathcal{N}$ in Example~\ref{example-1}. We can see that every edges in $G_{\mathcal{N}}$ is AC, but the path $\pi=(v_3, v_2,v_4)$ is not PC as $R_{34}$ is not contained in $R_{32}\circ R_{24}$.

In this paper we are concerned with BCNs defined over a particular \emph{constraint language} and we use \emph{constraint languages}, \emph{constraint classes} and \emph{sets of relations} interchangeably.

\begin{definition} \cite{jeavons1998constraints}
Let  ${\D}=\{D_1, \ldots, D_n\}$ be a set of domains. An $n$-ary relation $R$ over ${\D}$ is a subset of $D_1 \times \ldots \times D_n$. For any tuple $t\in R$ and any $1\le i\le n$, we {denote}  by $t[i]$ the value in the $i$-th coordinate position of $t$ and write $t$ as $\langle t[1],\ldots,t[n] \rangle$.
\end{definition}

\begin{definition} \cite{jeavons1998constraints}\label{dfn:definable}
Given a set of binary relations $\Gamma$, we write $\Gamma^+$ for the set of relations that can be obtained from $\Gamma$ using some sequence of the following operations: 
\begin{itemize}
\item \emph{Cartesian product}, \textit{i.e.}, for $R_1,R_2\in \Gamma, R_1 \times R_2=\{\langle t_1,t_2 \rangle \mid t_1\in R_1,t_2\in R_2\}$, 
\item \emph{equality selection}, \textit{i.e.}, for $R \in \Gamma, {\tau_{i=j}(R)=\{t\in R \mid t[i]=t[j]\}}$, and
\item \emph{projection}, \textit{i.e.}, for $R \in \Gamma, \pi_{i_1,\cdots,i_k}(R)=\{\langle t[i_1],\cdots,t[i_k] \rangle \mid t\in R\}$. 
\end{itemize}

\black A relation $R$ is said to be \emph{definable} in $\Gamma$ if $R\in {\Gamma^+}$\black {, and a} set of binary relations $\Gamma$ is said to be \emph{complete} if every binary relation definable in $\Gamma$ is also contained in $\Gamma$. The \emph{completion} of $\Gamma$, written as $\Gamma^c$, is the  set of all binary relations contained in $\Gamma^+$.
\end{definition}

The following lemma {asserts} that a complete set of binary relations $\Gamma$ is closed under the operations that are used to achieve PC and AC.

\begin{lemma} \cite{cohen2006_CP_handbook} \label{lem:closed_under_ac}
{
Let $\Gamma$ be a complete set of binary relations over a domain $D$. Suppose $R,S$ are binary relations and $T$ a unary relation, all in $\Gamma$. Then 
$R\cap S$, $R\circ S$, and 
$T'=\{a \in D \mid \langle a,b\rangle\in R,\ b\in T\}$
are also all in $\Gamma$.
}
\end{lemma}

Let $\Gamma$ be a set of binary relations. A BCN $\mathcal{N}=\la V,{\D},C \ra$ is \emph{defined over} (or, simply, \emph{over}) $\Gamma$ if $R\in \Gamma$ for every constraint $(s, R)$ in $C$. 
The constraint satisfaction problem (CSP) of $\Gamma$, \black {denoted by} $\csp(\Gamma)$, is \black {the problem of} deciding the consistency of BCNs defined over $\Gamma$. $\csp(\Gamma^+)$ is log-space reducible to $\csp(\Gamma)$ \cite{cohen2006_CP_handbook}.

 A set of binary relations $\Gamma$ is \emph{weakly closed under singletons}, if $\{\langle a,b \rangle\} \in \Gamma^+$ for any $R \in \Gamma$ and any $\langle a,b \rangle \in R$. 
 

In this paper we often assume that the constraint languages are \emph{complete} and \emph{weakly closed under singletons}. We will see that this is not very restrictive as, for any set  $\Gamma$ of binary relations  that is closed under a majority operation $\varphi$, the completion $\Gamma^c$ of $\Gamma$ is also closed under $\varphi$ \cite{jeavons1998constraints} and weakly closed under singletons (cf. Proposition~\ref{majority-is-wcus}).


\begin{algorithm}[t]  
   \DontPrintSemicolon
   \footnotesize
   \SetAlCapFnt{\footnotesize}
   \SetAlCapNameFnt{\footnotesize}
   \SetAlFnt{\footnotesize}
   \SetKwInOut{Input}{Input}

   \SetKwInOut{Output}{Output}

   \Input{A binary constraint network $\mathcal{N}=\la V,{\D},C \ra$; \newline an ordering $\prec$ $=$ $(v_1,\ldots,v_n)$ on $V$.}
   \Output{An equivalent \black {sub}network that is strongly {DPC} relative to $\prec$, or ``Inconsistent''. }
   \BlankLine

        \For{$k\gets n$ \KwTo $1$}{
        
            \For{{$i < k$} \emph{with} $R_{ik}\in C$}{
            
                   $D_i \gets D_i \cap R_{ki}({D_k})$;
                   
                   \If{{$D_i= \varnothing$}}{\label{ln:dpc_ac}
                   
                   \Return{{``Inconsistent''}}
                   
                   }
                   }

                   \For{{$i,j < k$} \emph{with} $R_{ik},R_{jk} \in C$}{
            
                    \If{$R_{ij} \not\in C$}{
                         $R_{ij}\gets D_i \times D_j$\label{algo1:ln8};
                         
                         $C \leftarrow C \cup \{R_{ij}\}$;
                    }
            
                    $R_{ij} \leftarrow R_{ij} \cap (R_{ik} \circ R_{kj})$;
                
                     \If{$R_{ij} = \varnothing$}{
                
                          \Return{``Inconsistent''};
                    }
             }
          
       }
\Return{$\mathcal{N}$}.
\caption{\dpc}\label{algorithm_strongppc}
\end{algorithm}

\section{The Strong {Directional PC} Algorithm}\label{sec:DPC_Algo}

\begin{figure}[t]
  \centering
  \subcaptionbox{The constraint graph $G_{\N}$ of a graph \\ coloring problem $\N$.}[.45\linewidth]{\includegraphics[width=.25\linewidth]{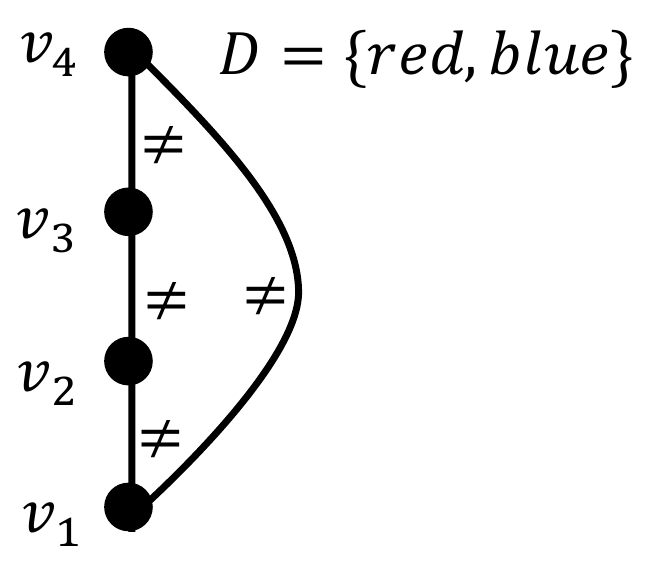}}
  \subcaptionbox{The constraint graph obtained by applying Algorithm~\ref{algorithm_strongppc} to $(\N,\prec)$ where $\prec$ $=$ $(v_1,v_2,v_3,v_4)$.\label{fig:coloring:b}}[.4\linewidth]{\includegraphics[width=.16\linewidth]{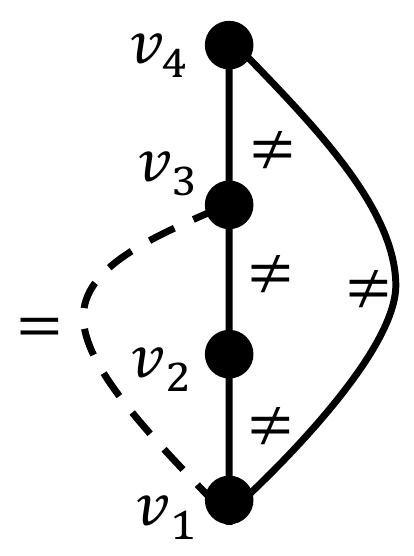}}
  \caption{A graph coloring problem with domain $D_i=\{red, blue\}$ for $i=1,2,3,4$ \cite{dechter2003constraint}.}\label{fig:coloring}
\end{figure}


This section recalls the notions of \emph{directional arc-consistency} (DAC) and \emph{directional path-consistency} (DPC), and the \emph{strong DPC} enforcing algorithm of Dechter and Pearl \cite{dechter2003constraint}.

\begin{definition} \cite{dechter2003constraint}
Let $\mathcal{N}=\la V,{\D},C\ra$ be a BCN and $\prec=(v_1,\ldots,v_n)$ an ordering of the variables in $V$. We say that $\mathcal{N}$ is \emph{directionally arc-consistent (DAC)} relative to $\prec$ if $v_i$ is arc-consistent relative to $v_k$ for all  $k> i$ with $R_{ik}\in C$. Similarly, $\mathcal{N}$ is \emph{directionally path-consistent (DPC)} relative to $\prec$ if, for any $i\not=j$ with $R_{ij}\in C$,  $(v_i,v_j)$ is path-consistent relative to $v_k$ for all $k>i,j$ {whenever} $R_{ik}, R_{jk}\in C$. Meanwhile, $\mathcal{N}$ is \emph{strongly DPC} relative to $\prec$ if it is both DAC and DPC relative to $\prec$.
\end{definition}

The strong DPC algorithm is presented as Algorithm~\ref{algorithm_strongppc}. In comparison with traditional PC algorithms \cite{chmeiss1998efficient},
a {novelty} of this {single pass} algorithm is its explicit {reference to \black {the constraint} graph of the input constraint network}. As only Line~\ref{algo1:ln8} may require extra working space, Algorithm~\ref{algorithm_strongppc} has a {very low space complexity in practice}. Further, Algorithm~\ref{algorithm_strongppc} runs  in $O(w^*(\prec) \cdot e \cdot (\alpha + \beta))$ time \cite{dechter2003constraint}, where $e$ is the number of edges of the output constraint graph, $w^*(\prec)$ is the induced width  of the ordered graph along $\prec$, and $\alpha,\beta$ are the runtimes of relational intersection and composition respectively. Note that $w^*(\prec) \le n$ and $\alpha, \beta$ are bounded by $O(d)$, where $d$ is the largest domain size.


\begin{proposition} \cite{dechter2003constraint}\label{lem:peo}
Let $(\mathcal{N},\prec)$ be an input to Algorithm~\ref{algorithm_strongppc}, {where $\prec$ $=$ $(v_1,\ldots,v_n)$}. Suppose $\mathcal{N}'=\la V, {{\D}'}, C'\ra$ is the output. Then 


\begin{itemize}
\item [(i)] {$G_{\mathcal{N}'}$ is triangulated and $\prec^{-1}$, the inverse of $\prec$, is a PEO of  $G_{\mathcal{N}'}$}; 
\item [(ii)] $\mathcal{N}'$ is equivalent to $\mathcal{N}$ and {strongly DPC  relative to $\prec$}. 
\end{itemize}
\end{proposition}

Let $\Gamma$ be a set of binary relations. We say that Algorithm~\ref{algorithm_strongppc} \emph{decides} $\csp(\Gamma)$ if, for any given BCN $\mathcal{N}$ in $\csp(\Gamma)$
and any ordering $\prec$ of variables of $\mathcal{N}$,
Algorithm~\ref{algorithm_strongppc} returns ``Inconsistent'' iff $\mathcal{N}$ is inconsistent.

The following corollary follows directly from Proposition~\ref{lem:peo}.
\begin{corollary}\label{coro:dpc}


Let $\Gamma$ be a complete set of binary relations. Then the following two conditions are equivalent:
\begin{itemize}
\item [(i)] Algorithm~\ref{algorithm_strongppc} decides $\csp(\Gamma)$. 
\item [(ii)] Let $\mathcal{N}$ be a non-trivially inconsistent BCN in $\csp({\Gamma})$. Suppose $\N$'s constraint graph $G_{\N}$ is triangulated and let $\prec^{-1}$ $=$ $(v_n,\ldots,v_1)$ be a PEO of it. Then $\mathcal{N}$ is consistent if $\mathcal{N}$ is strongly DPC relative to $\prec$.
\end{itemize}
\end{corollary}

\begin{example}
The graph coloring problem $\N$ with domains $\{red,blue\}$ depicted in Fig.~\ref{fig:coloring} is taken from \cite{dechter2003constraint} and can be decided by Algorithm~\ref{algorithm_strongppc}. After applying Algorithm~\ref{algorithm_strongppc} to $(\N,\prec)$, where $\prec=(v_1,v_2,v_3,v_4)$, a solution can be obtained along $\prec$ in a backtrack-free {manner} (see Fig.~\ref{fig:coloring:b}). 

\end{example}

\section{{Directional PC} and Variable Elimination}\label{sec:DPC_VE}

This section characterizes {the binary constraint languages $\Gamma$ such that $\csp(\Gamma)$} can be decided by \dpc. We observe that \dpc achieves (strong) DPC using the idea of \emph{variable elimination} \cite{dechter2003constraint}: it iterates variables {along} the ordering $\prec^{-1}$, and propagates the constraints of a variable $v_k$ to subsequent variables in the ordering with the update rule $R_{ij} \leftarrow R_{ij}\cap (R_{ik} \circ R_{kj})$, as if $v_k$ is `\emph{eliminated}'. 

The following definition formalizes the process of elimination.
\begin{definition} \label{dfn:vep_network}
Let $\mathcal{N}=\la V,{\D},C \ra$ be a BCN with $V=\{ v_1,...,v_n \}$ and ${\D}=\{ D_1,...,D_n\}$. For  a variable $v_x$ in $V$, 
let $E_x  =\{R_{ix}\mid R_{ix}\in C\}$.
The network obtained after $v_x$ is eliminated from $\mathcal{N}$, written as  
\begin{align*}
\mathcal{N}_{-x}  =\la V\setminus\{v_x\},\{D'_1, ..., D'_{x-1}, D'_{x+1}, ..., D'_n\},C'\ra,
\end{align*}
is defined as follows:

\begin{itemize}

\item If $E_x=\{R_{ix}\}$, we set $C' = C \setminus E_x$ and let
\begin{align}
D'_j &= \left\{\begin{array}{cl} \label{eqq:1neighbor}
        D_i\cap R_{xi}(D_x), & \quad \text{if } j=i\\ 
        D_j, & \quad \text{otherwise} 
        \end{array}
        \right.
\end{align}
\item If $|E_x|\not=1$,  we set $D'_j=D_j$ for all $j\not=x$, and let 
\begin{align*}
C' = (C\setminus E_x)\cup \{R_{ix}\circ R_{xj} \cap R_{ij} \mid R_{jx},R_{ix}\in E_x, i \ne j\}.
\end{align*}
\end{itemize}
\end{definition}

$R_{ij}$ is assumed to be $D_i\times D_j$ if $R_{ij}\not\in C$. 

Fig.~\ref{fig:vep} illustrates the elimination process.

\begin{figure}[t!]
\centering
\includegraphics[width=270pt]{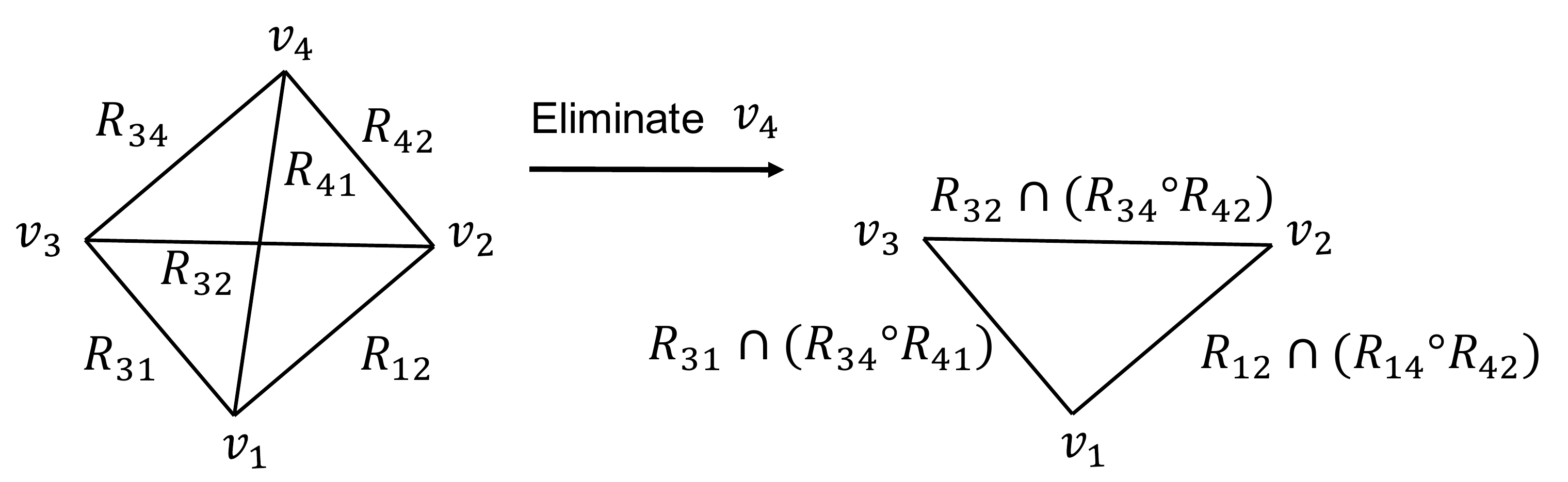}
\caption{{\black {Two binary} constraint \black {networks} $\mathcal{N}$ and $\mathcal{N}_{-1}$.}}  \label{fig:vep}
\end{figure}
\begin{definition}\label{dfn:vep}
A BCN $\mathcal{N}=\la V,{\D},C\ra$ is said to have the \emph{variable elimination property} (VEP), if, for any $v_x$ in $V$, {every} solution of $\mathcal{N}_{-x}$ can be extended to a solution of $\mathcal{N}$.  

$\mathcal{N}$ is said to have \black {\emph{weak}} VEP, if, for any $v_x$ in $V$ such that $v_x$ is AC relative to all relations in $E_x$, {every} solution of $\mathcal{N}_{-x}$ can be extended to a solution of $\mathcal{N}$. 


A set of binary relations $\Gamma$ is said to have (weak) VEP if every BCN in $\csp(\Gamma)$ has (weak) VEP. Such a set of binary relations $\Gamma$ is also called a (weak) VEP class. 
\end{definition}

It is easy to see that, if a BCN (a set of binary relations) has VEP, then it also has weak VEP. The following example explains why we should take special care when eliminating variables {with} only one successor in Eq.~\eqref{eqq:1neighbor}.


\begin{figure}[htb]
  \begin{center}    
  \includegraphics[width=0.1\textwidth]{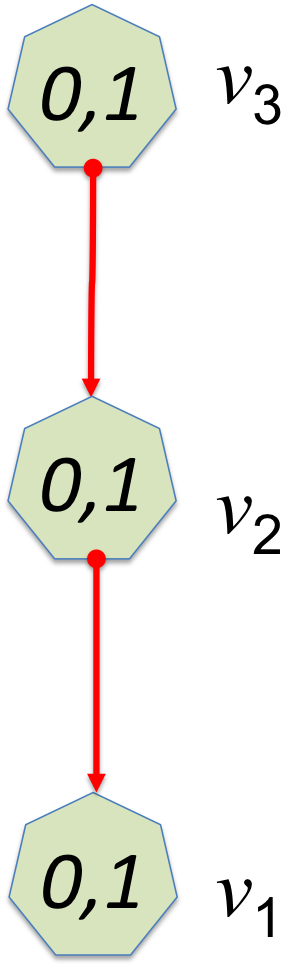}
  \end{center}
  \caption{A constraint graph that is a chain.}\label{fig:chain_constraint_graph}
\end{figure}

\begin{example}\label{ex:linear_csp}
Let $\mathcal{N}=\la V,{\D},C\ra$ be a BCN defined \black {by} $V=\{v_1,v_2,v_3\}$, $D_1=D_2=D_3=\{0,1\}$, \black {and} $C=\{((v_3,v_2), R),$ $((v_2,v_1), R)\}$ with $R=\{(1,0)\}$ (see Fig.~\ref{fig:chain_constraint_graph}). Suppose we do not have the operation specified in \eqref{eqq:1neighbor} and $\prec$ $=$ $(v_3,v_2,v_1)$ is the variable elimination ordering. Let $\mathcal{N}_{-3}$ be the restriction of $\mathcal{N}$ to $\{v_1,v_2\}$. Then $\mathcal{N}_{-3}$ has a unique solution $\sigma$ but it cannot be extended to a solution of $\mathcal{N}$.  
\end{example}


\begin{proposition} \label{proposition_VEP_DPCP}
Let $\Gamma$ be a complete set of binary relations {that is weakly closed under singletons}. Then \dpc decides $\csp(\Gamma)$  {iff} $\Gamma$ has VEP.
\end{proposition}
\begin{proof}
We address the `if' part first. 
Assume that $\Gamma$ has VEP, and let $\mathcal{N}=\la V,{\D},C\ra$ be a network in $\csp(\Gamma)$ that is non-trivially inconsistent and strongly DPC relative to $\prec=(v_1,\ldots,v_n)$, where $G_{\N}$ is triangulated and $\prec^{-1}$ is a PEO of it. We show that $\mathcal{N}$ is consistent. Let $V_i=\{v_1,\ldots,v_{i}\}$ and $\mathcal{N}{\mid_{V_i}}$ be the restriction of $\mathcal{N}$ to $V_i$. We claim that $\mathcal{N}{\mid_{V_i}}$ is consistent for $k=1,\ldots,n$ and prove the claim by induction on $k$. First, since $\mathcal{N}$ is strongly DPC relative to $\prec$,
$D_1$ is not empty and there is an $a_1\in D_1$. 
Then, $\mathcal{N}{\mid_{V_1}}$ is consistent and has a solution $\sigma_1=\la a_1\ra$. Further, suppose that $\mathcal{N}{\mid_{V_i}}$ is consistent and $\sigma_i =\la a_1,a_2,...,a_i\ra$ is a solution of $\mathcal{N}{\mid_{V_i}}$. We show that $\sigma_i$ can be extended to a solution $\sigma_{i+1}=\la a_1,\ldots,a_i,a_{i+1} \ra$ of $\mathcal{N}{\mid_{V_{i+1}}}$. We consider three subcases: (i) If $E_{i+1}$ is empty, then for any $a_{i+1} \in D_{i+1}$, $\sigma_{i+1} =\la a_1,a_2,...,a_i,a_{i+1}\ra$ is a solution of $\mathcal{N}{\mid_{V_{i+1}}}$ because there is no constraint between $v_i$ and $v_{i+1}$. (ii) If {$E_{i+1}=\{R_{j,i+1}\}$} is a singleton, then, since $\mathcal{N}$ is DAC relative to  $\prec$, we have $a_{i+1}\in D_{i+1}$ such that $\la a_i,a_{i+1}\ra \in R_{i,i+1}$ and $\sigma_{i+1} =\la a_1,a_2,...,a_i,a_{i+1}\ra$ is a solution of $\mathcal{N}{\mid_{V_{i+1}}}$. 
(iii) If $E_{i+1}$ contains more than one variable, for every pair of distinct variables $(v_x,v_y)$ in $V_i$ with $R_{x,i+1}, R_{y,i+1} \in E_{i+1}$, we know $R_{xy}\in C$ because  {$\prec^{-1}$ is a PEO of $G_\mathcal{N}$}. Moreover, since $(v_x,v_y)$ is PC relative to $v_{i+1}$, we have $R_{xy} \subseteq R_{x,i+1}\circ R_{i+1,y}$. Then since $\Gamma$ has VEP and $\mathcal{N}{\mid_{V_i}}$ is indeed the same network as the one obtained by eliminating $v_{i+1}$ from $\mathcal{N}{\mid_{V_{i+1}}}$,  by {Definition}~\ref{dfn:vep}, $\sigma_i$ can be extended to a solution $\sigma_{i+1}$ of $\mathcal{N}{\mid_{V_{i+1}}}$. Thus, $\mathcal{N}$ is consistent. By {Corollary}~\ref{coro:dpc}, \dpc decides $\csp(\Gamma)$.

Next, we address the `only if' part. Assume that \dpc decides $\csp(\Gamma)$. We show that $\Gamma$ has VEP. Let $\mathcal{N}=\la V,{\D},C\ra$ be a non-trivially inconsistent network in $\csp(\Gamma)$. Given $v_x\in V$, we show that every solution of $\mathcal{N}_{-x}$ can be extended to $\mathcal{N}$. Without loss of generality, we assume that $x=n$. Let $\sigma=\langle a_1,\ldots,a_{n-1} \rangle$ be a solution of $\mathcal{N}_{-n}$, and $E_n=\{R_{in}\mid R_{in}\in C\}$. By the definition of $\N_{-n}$, for any $R_{in},R_{jn}\in E_n(i\ne j)$, we have $\la a_i,a_j\ra\in R_{in}\circ R_{nj} \cap R_{ij}$. We then construct a new problem $\mathcal{N}'=\la V,{\D'},C' \ra$ in $\csp(\Gamma)$, where  {${\D'}=\{D'_1,...,D'_{n-1}, D_n\}$ with $D'_i=\{a_i\}$ for $1\leq i<n$ and} $C'=\{\{\langle a_i,a_j \rangle\} \mid 1\leq i\not=j <n\} \cup E_n$. Clearly, $\sigma$ is also a solution of $\mathcal{N}'_{-n}$ and $\mathcal{N}'_{-n}$ is strongly PC and, hence, strongly DPC relative to the ordering $(v_1,\ldots,v_{n-1})$. Further, since $\la a_i,a_j\ra\in R_{in}\circ R_{nj} \cap R_{ij}$ for any $R_{in},R_{jn}\in E_n(i\ne j)$, we have that $\N'$ is strong DPC relative to $\prec$ $=$ $(v_1,\ldots,v_n)$. As $G_{\N'_{-n}}$ is complete, $G_{\N'}$ is triangulated with $\prec^{-1}$ being a PEO of it. As \dpc decides $\csp(\Gamma)$ and $\mathcal{N}'\in \csp(\Gamma)$, by {Corollary}~\ref{coro:dpc}, $\mathcal{N}'$ is consistent and has a solution \black {that extends $\sigma$ and is also a solution of $\mathcal{N}$}. This shows that $\mathcal{N}$ is consistent and, hence, $\Gamma$ has VEP.\qed
\end{proof}

Therefore, if $\mathcal{N}=\la V,{\D},C\ra$ is defined over a complete VEP class,  then \dpc can decide it. Note that in the above proposition we require $\Gamma$ to be complete. This is important; for example,  every \emph{row-convex constraint} \cite{van1995minimality} network has VEP (cf. the proof of \cite[Theorem~1]{zhang2009solving}) and, hence, the class of row-convex constraints has VEP. However, \dpc does not decide the consistency problem over the row-convex constraint class because it was shown to be NP-hard (cf. e.g.~\cite{KongLLL15}).


VEP is closely related to the Helly property, defined as follows.

\begin{definition}\label{dfn:helly}
A set $\Gamma$ of binary relations over $\D=\Set{D_1,...,D_n}$ is said to have  the \emph{Helly property} if for any $k>2$ binary relations, $R_i \subseteq D_{u_i}\times D_{u_0}(1\le i\le k, 1\le u_i\ne u_0\le n)$, in $\Gamma$, and for any $k$ values, $a_i \in D_{u_i}(1\leq i\leq k)$, such that $R_i(a_i)=\Set{b \in D_{u_0} \mid \la a_i,b\ra \in R_i}$ is nonempty, we have $\bigcap_{i=1}^k R_i(a_i)\not=\varnothing$ iff $R_{i}(a_i) \cap R_j(a_j) \not=\varnothing$ for any $1\leq i\not=j\leq k$.
\end{definition}

\begin{figure}[htb]
  \begin{center}    
  \includegraphics[width=0.3\textwidth]{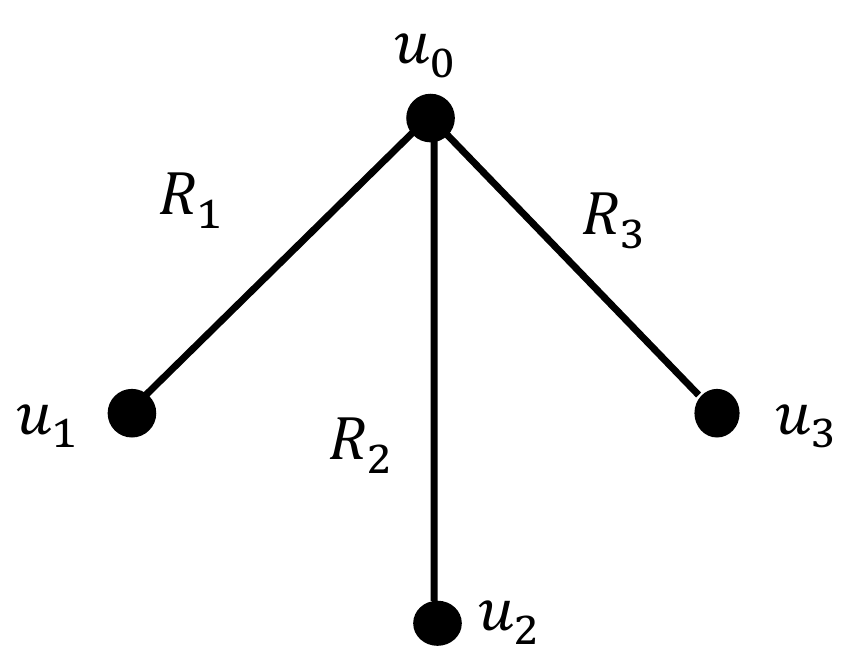}
  \end{center}
  \caption{An illustration of Example~\ref{example:Helly}.}\label{fig:Helly}
\end{figure}

\begin{example}\label{example:Helly}
Let $D_{u_0} = \{a,b,c,d\}, D_{u_1}=\{e\}, D_{u_2}=\{f\}, D_{u_3}=\{g\}$ and $R_1=\{\la e,a \ra, \la e,b \ra, \la e,c \ra\}, R_2=\{\la f,b \ra, \la f,c \ra, \la f,d \ra\}, R_3=\{\la g,c \ra, \la g,d \ra, \la g,a \ra\}$. See Fig.~\ref{fig:Helly} for an illustration. Then $\Gamma=\{R_1,R_2,R_3\}$ over $\D=\{D_{u_0},D_{u_1},D_{u_2},D_{u_3}\}$ has the Helly property.
\end{example}


\begin{theorem}\label{prop:helly+}
A set of binary relations $\Gamma$ has VEP iff it has the Helly property.
\end{theorem}
\begin{proof}
Suppose $\Gamma$ has VEP, we show $\Gamma$ has the Helly property. Let ${\D}=\{D_1,...,D_n\}$ be the set of domains related to relations in $\Gamma$. Suppose 
$R_i\subseteq D_{u_i}\times D_{u_0}$ ($1\leq i\leq k$, $1\le u_i\ne u_0 \le n$) are $k>2$ binary relations in $\Gamma$ and $a_i\in D_{u_i}$ ($1\leq i\leq k$) are values such that $\varnothing\not=R_i(a_i) \subseteq D_{u_0}$. Suppose $R_i(a_i)\cap R_j(a_j)$ is nonempty for any $i,j$ with $1\leq i\not=j\leq k$. We show that $\bigcap_{i=1}^k R_i(a_i)$ is nonempty. To this end, we construct a BCN $\mathcal{N}=\la V, {{\D}'},  C\ra$ over $\Gamma$ with $V=\{v_1,...,v_k,v_{k+1}\}$, ${\D'}=\{D_{u_1}, ..., D_{u_k}, D_{u_0}\}$, and $C=\{R_{i,k+1}\mid 1\leq i\leq k\}$, where $R_{i,k+1}=R_i$. Consider $\mathcal{N}_{-(k+1)}$. As $R_i(a_i)\cap R_j(a_j)\not=\varnothing$, we have $\la a_i,a_j\ra\in R_{i,k+1}\circ R_{k+1,j}$. This shows that $\sigma=\la a_1,\ldots,a_k\ra$ is a solution of $\mathcal{N}_{-(k+1)}$. {Since} $\Gamma$ and, hence, $\mathcal{N}$ have VEP, $\mathcal{N}$ has a solution that extends $\sigma$. Hence there exists $a\in {D_{u_0}}$ such that $a\in R_{i,k+1}(a_i)$ for every $1\leq i\leq k$. Thus $\bigcap_{i=1}^k R_{i}(a_i)\not=\varnothing$. This proves that $\Gamma$ has the Helly property.

Suppose $\Gamma$ has the Helly property, we show $\Gamma$ has VEP. Let $\mathcal{N}=\la V,{\D},C\ra$ be a non-trivially inconsistent BCN defined over $\Gamma$ with $V=\{v_1,v_2,...,v_n\}$ and $C$ is a set of binary constraints $((v_i,v_j),R)$  with $R\in \Gamma$. Let $E_n=\{R_{in}\mid R_{in}\in C\}$. Assume $\sigma=\la a_1,a_2,...,a_{n-1}\ra$ is a solution of, say, $\mathcal{N}_{-n}$. We show that there exists $a_n\in D_n$ such that $\la a_1,...,a_{n-1}, a_n\ra$ is a solution of   $\mathcal{N}$. If $E_n$ is empty, we can take any $a_n$ from $D_n$ which is nonempty since $\mathcal{N}$ is non-trivially inconsistent; if $E_n$ contains only one constraint, say, $((v_i,v_n),R_{in})$, by $a_{i}\in D'_i=D_i\cap R_{ni}(D_n)$, there exists $a_n\in D_n$ such that $\la a_i,a_n\ra\in R_{in}$; if $E_n$ contains $k\geq 2$ constraints and let them be $((v_{u_i},v_n), R_{u_in})$ $(1\leq i \leq k)$,  we have $\la a_i,a_j\ra\in R_{u_i u_j}\cap (R_{u_i n}\circ R_{n u_j})$ for $1\leq i\not=j \leq k$. Therefore, we have $R_{u_i n}(a_i)\cap R_{u_j n}(a_j)\not=\varnothing$ for $1\le i\not=j \le k$. By the Helly property of $\Gamma$,  we have $\bigcap_{i=1}^k R_{u_i n}(a_i) \ne \varnothing$. So we can take any $a_n\in \bigcap_{i=1}^k R_{u_i n}(a_i)$ so that $\la a_1,\ldots,a_{n-1},a_n \ra$ is a solution of $\mathcal{N}$. Therefore, $\Gamma$ has VEP.
~\hfill\qed
\end{proof}

The class of row-convex constraints \cite{van1995minimality} and the class of \emph{tree-convex} constraints \cite{ZhangY03} have the Helly property and, thus, they have VEP by Theorem~\ref{prop:helly+}.


\begin{proposition}\label{HellyPC}
Suppose $\Gamma$ is a set of binary relations that has the Helly property. Let $\N \in \csp(\Gamma)$. Suppose $\N$ is non-trivially inconsistent and $G_{\N}$ is triangulated with $\prec^{-1}=(v_n,\ldots,v_1)$ as a PEO of it. Then $\N$ is consistent if it is strongly DPC relative to $\prec$.
\end{proposition}
\begin{proof}

Let $\N \in \csp(\Gamma)$. Suppose $\N = \la V,\D,C \ra$ is non-trivially inconsistent and $G_{\N}$ is triangulated with $\prec^{-1}$ $=$ $(v_n,\ldots,v_1)$ being a PEO of it. Suppose $\N$ is strongly DPC relative to $\prec$. We show that $\N$ is consistent. Let $V_k=\{v_1,\ldots,v_k\}$ and $\N_k$ be the restriction of $\N$ to $V_k$. Since $\N$ is non-trivially inconsistent, we have that $\N_{1}$ is consistent. Suppose $\N_k$ is consistent, we show $\N_{k+1}$ is consistent. Let $\sigma=\la a_1,\ldots,a_k \ra$ be a solution of $\N_k$. Let $E_{k+1}=\{R_{i,k+1} \mid R_{i,k+1} \in C, i \le k \}$. Since $G_{\N}$ is triangulated and $\prec^{-1}=(v_n,\ldots,v_1)$ is a PEO of it, for any two different constraints $R_{i,k+1}, R_{j,k+1}\in E_{k+1}$, we have $R_{ij}\in C$. Further, since $\N$ is strongly DPC relative to $\prec$, we have $\la a_i,a_j \ra \in R_{i,k+1}\circ R_{k+1,j} \cap R_{ij}$. Thus, we have $R_{i,k+1}(a_i) \cap R_{j,k+1}(a_j) \ne \varnothing$ for any two different constraints $R_{i,k+1}, R_{j,k+1}\in E_{k+1}$. Since $\Gamma$ has the Helly property, we have $\bigcap_{R_{i,k+1}\in E_{k+1}} R_{i,k+1}(a_i) \ne \varnothing$. Therefore, $\sigma$ can be extended to a solution of $\N_{k+1}$ and $\N_{k+1}$ is consistent. By induction on $k$, we have that $\N$ is consistent.
~\hfill\qed
\end{proof}



\section{Majority-Closed Constraint Languages}\label{sec:majority-closed}

In this section we {characterize} {weak VEP classes}. We will show that a complete set of binary relations $\Gamma$ has weak VEP iff all relations in $\Gamma$ are closed under a \emph{majority} operation, which is defined as follows.

\begin{definition} \black {\cite{BulatovJ03}}\label{def_closure}
Let $\D=\{D_1,\ldots,D_n\}$ be a set of domains. A \emph{multi-sorted majority operation} $\varphi$ \black {on $\D$} is a set $\{\varphi_1,\ldots,\varphi_n\}$, where $\varphi_i$ is a one-sorted majority operation on $D_i$, i.e.,  a mapping from $D_i^3$ to $D_i$ such that $\varphi_i(e,d,d)=\varphi_i(d,e,d)=\varphi_i(d,d,e)=d$ for all $d,e$ in $D_i$. 

An $m$-ary relation $R\subseteq D_{i_1}\times...\times D_{i_m}$ with $i_1,...,i_m\in \Set{1,2,...,n}$ is said to be \emph{closed under $\varphi$} if  $\varphi(t_1,t_2,t_3)=\langle \varphi_{i_1}(t_1[1], t_2[1],t_3[1]), \ldots, \varphi_{i_m}(t_1[m],t_2[m], t_3[m]) \rangle$ is in $R$ for any $t_1,t_2,t_3 \in R$.  

A set of relations $\Gamma$ is said to be \emph{closed under $\varphi$} if every $R \in \Gamma$ is closed under $\varphi$.

\end{definition}
A set of relations $\Gamma$ is called a \emph{majority-closed} language if there exists a (multi-sorted) majority operation $\varphi$ such that every relation in $\Gamma$ is closed under $\varphi$. 

\subsection{Tree-Preserving Constraints}\label{subsec:majority}


The class of tree-preserving constraints is majority-closed.

\begin{definition}\cite{Kong2017}
An undirected graph structure can often be associated to a finite domain $D$ such that there is a bijection between the vertices in the graph and the values in $D$. If the graph is connected and acyclic, i.e. a tree, then we say it is a \emph{tree domain} of $x$, denote as $T=(D,E)$ where $E$ is a set of edges. A constraint $R_{ij}$ over tree domains $T_i=(D_i,E_i)$ and $T_j=(D_j,E_j)$ is called \emph{tree-preserving} if the image of every subtree in $T_i$ is a subtree in $T_j$.
\end{definition}


An example of tree-preserving constraint is shown in Fig.~\ref{fig:tree}. A CRC constraint is a special tree-preserving constraint where related tree domains are chains \cite{Kong2017}.

\begin{figure}[htb]
  \begin{center}    
  \includegraphics[width=0.3\textwidth]{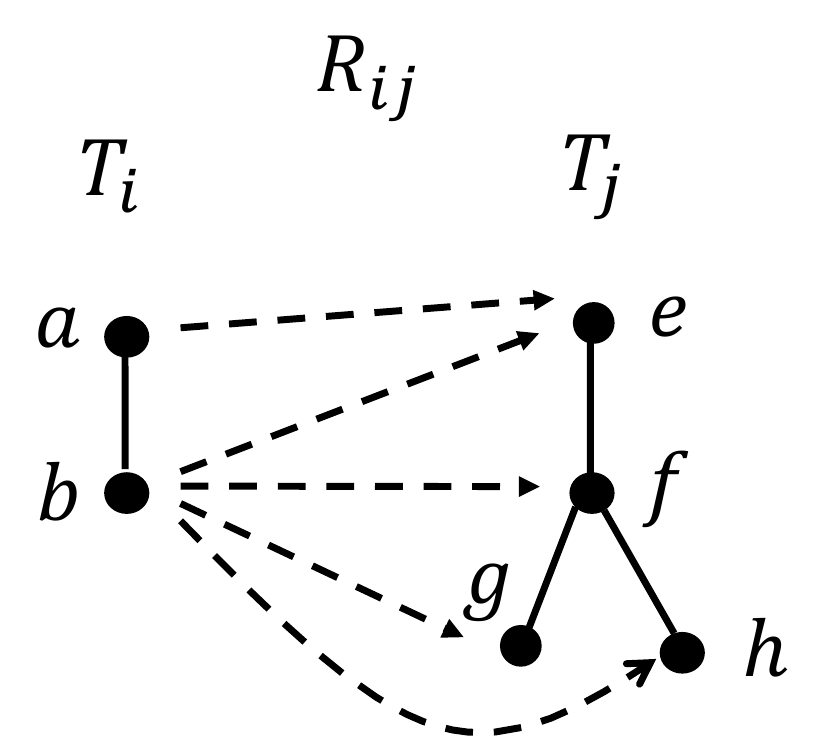}
  \end{center}
  \caption{An example of tree-preserving constraint $R_{ij}$ with $D_i=\{a,b\}$ and $D_j=\{e,f,g,h\}$, where a dashed arrow from a node $u$ in $T_i$ to a node $v$ in $T_j$ indicates that $\la u,v \ra\in R_{ij}$.}\label{fig:tree}
\end{figure}

\begin{definition}
Let $T_i$ be a nonempty tree domain for a variable $v_i$. The \emph{standard} majority operation $m_i$ on $T_i$ is defined as:

$$(\forall a,b,c\in T_i) \text{\quad} m_i(a,b,c)=\pi_{a,b}\cap \pi_{b,c}\cap \pi_{a,c}, $$
where $a,b,c$ are not necessarily distinct and $\pi_{u,v}$ denotes the unique path from $u$ to $v$ in $T_i$.

The following result establishes the connection between tree-preserving constraints and majority-closed constraints.

\begin{theorem}\cite{Kong2017}
Let $T_i$ and $T_j$ be two nonempty tree domains and $m_i$ and $m_j$ their standard majority operations. Suppose $R_{ij} \subseteq T_i \times T_j$ is a nonempty constraint such that both $R_{ij}$ and $R_{ji}$ are arc-consistent. Then $R_{ij}$ is closed under $\{m_i,m_j\}$ iff both $R_{ij}$ and $R_{ji}$ are tree-preserving w.r.t. $T_i$ and $T_j$.
\end{theorem}

\end{definition}

\subsection{Weak VEP Classes and Majority-Closed Classes}

We first study a few properties of majority-closed classes. 

\begin{proposition}\label{majority-is-wcus}
Let $\Gamma$ be the set of {binary} relations   that {is} closed under a multi-sorted majority operation $\varphi=\{\varphi_1,...,\varphi_n\}$ on ${\D}=\{D_1,...,D_n\}$. Then $\Gamma$ is weakly closed under singletons.
\end{proposition}
\begin{proof}
Suppose $R$ is a relation in $\Gamma$ and $\langle a,b \rangle \in R$  $\subseteq D_i\times D_j$. We show that $\{\langle a,b \rangle\}$ is closed under $\varphi$. For any $t_1,t_2,t_3\in \{\langle a,b \rangle\}$, we have $t_1=t_2=t_3=\langle a,b\rangle$, and, hence, $\varphi(t_1,t_2,t_3)=\langle \varphi_{i}(a,a,a),\varphi_{j}(b,b,b)\rangle =\langle a, b\rangle$. This shows that $\{\langle a,b \rangle\}$ is closed under $\varphi$ and, hence, a relation in $\Gamma$. 
~\hfill\qed\end{proof}

Majority-closed relations are \emph{decomposable}.
\begin{definition}
An $m$-ary relation $R$ is said to be \emph{$r$-decomposable} if, for any $m$-tuple $t$,  $t\in R$ if $\pi_I(t)\in \pi_I(R)$ for all $I=(i_1,\ldots,i_k)$ (a list of indices from $\{1,\ldots,m\}$) with $k\le r$, where $\pi_I(t)=\langle t[i_1],...,t[i_k]\rangle$ and $\pi_I(R)=\{\langle t[i_1],...,t[i_k]\rangle \mid t\in R\}$.
\end{definition}


\begin{theorem}\cite{jeavons1998constraints} \label{theorem-jeavons}
Let $\Gamma$ be a set of {binary} relations over a set of finite domains ${\D}=\{D_1,\ldots,D_n\}$. The following {statements} are equivalent:
\begin{itemize}
 \item[(1)] {$\Gamma$ is a majority-closed constraint language.}
 \item[(2)] Every $R \in \Gamma^+$ is $2$-decomposable.
 \item[(3)] For every $\mathcal{N} \in \csp({\Gamma})$, establishing strong PC in $\mathcal{N}$ ensures global consistency. 
\end{itemize}
\end{theorem}

\black {Finally}, we show that complete weak VEP classes are majority-closed classes.
\begin{theorem} \label{VEP_SPCP}
Let $\Gamma$ be a complete set of binary relations over a set of finite domains ${\D}=\{D_1$, ..., $D_n\}$. Then $\Gamma$ has weak VEP iff it is {a majority-closed class.} 
\end{theorem}

\begin{proof}
We first deal with the `only if'
part. Suppose that $\Gamma$ is a complete set of binary relations that has weak VEP. By \black {Theorem}~\ref{theorem-jeavons}, we only need to show that for every BCN in $\csp({\Gamma})$, establishing strong PC ensures global consistency. Let $\mathcal{N}^0$ be a network in $\csp({\Gamma})$ and suppose $\mathcal{N}=\la V,{\D},C\ra$ is the network obtained by enforcing strong PC on $\mathcal{N}^0$. Since $\Gamma$ is complete and thus closed under operations {for achieving} AC and PC by Lemma~\ref{lem:closed_under_ac},  $\mathcal{N}$ is also a problem in $\csp({\Gamma})$. Suppose $\mathcal{N}$ is non-trivially inconsistent. We show that any partial solution of $\mathcal{N}$ can be extended to a solution of $\mathcal{N}$. 

Suppose $V'=\{v_{1},\ldots,v_{m-1}\} \subset V$ and $\sigma= \langle a_1,\ldots,$ $ a_{m-1} \rangle$ is a solution of $\mathcal{N}{\mid_{V'}}$, which is the restriction of $\mathcal{N}$ to $V'$. Assume further that 
$v_{m}$ $\not\in V'$ is a new variable and let $V''=V'\cup\{v_{m}\}$. We show that $\sigma$ can be consistently extended to $\mathcal{N}{\mid_{V''}}$, the restriction of $\mathcal{N}$ to $V''$. Because $\mathcal{N}$ is strongly PC, $\mathcal{N}{\mid_{V''}}$ is strongly PC as well. In particular, $v_i$ is AC relative to  $v_m$ for any $R_{im}$ in $C$, and $R_{ij}$ is PC relative to $v_m$ ({\it i.e.}, $R_{ij}\subseteq R_{im}\circ R_{mj}$) for any $i\not=j$  such that both $R_{im}$ and $R_{jm}$ are in $C$. By Definition~\ref{dfn:vep_network},  $\mathcal{N}{\mid_{V'}}$ is the same as $(\mathcal{N}{\mid_{V''}})_{-m}$, viz., the network obtained by eliminating $v_m$ from $\mathcal{N}{\mid_{V''}}$. Moreover, since $\mathcal{N}$ and, hence, $\mathcal{N}{\mid_{V''}}$ are AC, $v_m$ is AC relative to all constraints $R_{im}$ that are in $C$. By the assumption that $\Gamma$ has weak VEP,  $\sigma$ can be consistently extended to $v_m$. {Following} this reasoning, we will find a solution of $\mathcal{N}$ that extends $\sigma$.

Next, we consider the `if' part. Suppose that $\Gamma$ is a complete set of binary relations that is closed under some multi-sorted majority operation $\varphi=\{\varphi_1,\ldots,\varphi_n\}$ on $\D$. Let $\mathcal{N}=\la V,{\D},C\ra$ be a problem in $\csp({\Gamma})$ {and $v_x$ a variable in $V$}. Let $E_x=\{R_{ix}\mid R_{ix}\in C\}$, and $\mathcal{N}_{-x}=\la V\setminus\{v_x\},{\D},C'\ra$, where $C'=(C\cup \{R_{ix}\circ R_{xj} \cap R_{ij} \mid R_{jx},R_{ix}\in E_x\}) \setminus E_x$. {Suppose that $v_x$ is AC relative to all relations in $E_x$}. We only need to  show that {any solution of $\mathcal{N}_{-x}$ can be extended to a solution of $\mathcal{N}$}. We prove this by contradiction.

\begin{figure}[t!] 
\centering
\includegraphics[width=240pt]{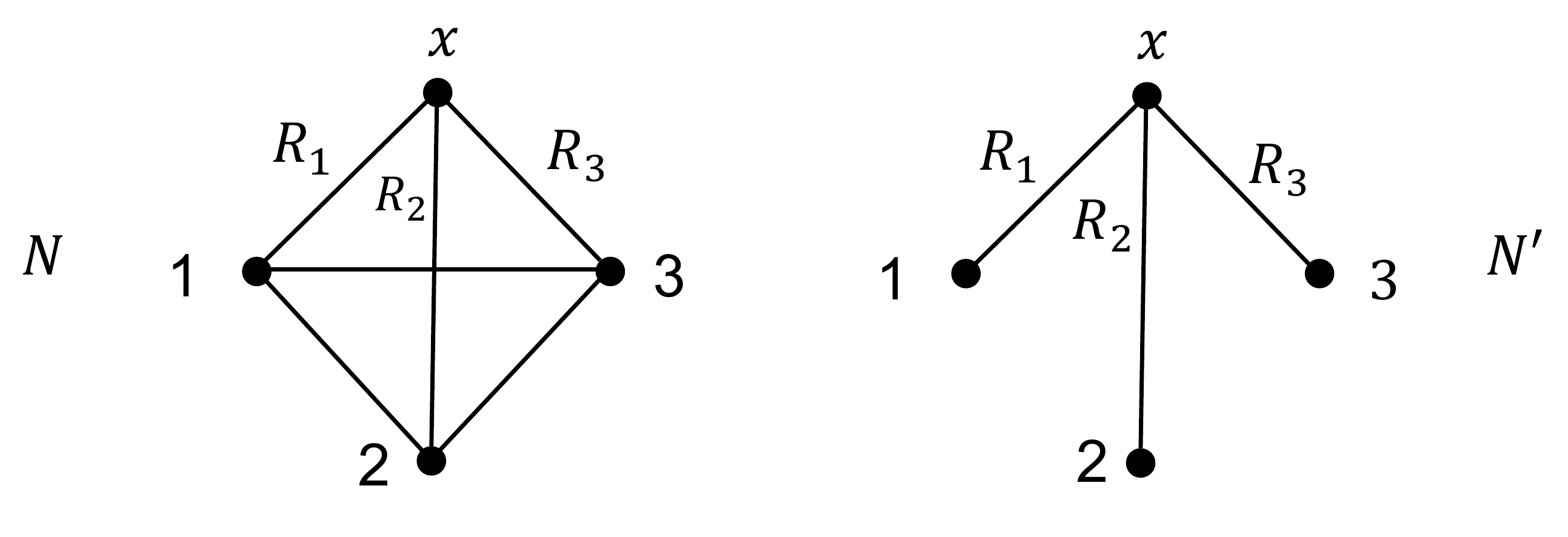}
\caption{Illustration of proof of Theorem~\ref{VEP_SPCP}.}\label{proof3}
\end{figure}

{Let $\sigma$ be a solution of $\mathcal{N}_{-x}$. Assume that $\sigma$ cannot be extended to a solution of $\mathcal{N}$.} Therefore, $E_x$ cannot be empty, otherwise $\sigma$ can be trivially extended to a solution of $\mathcal{N}$. {The case where $E_x=\{R_{ix}\}$ is a singleton is also impossible, as by \eqref{eqq:1neighbor}, $v_i$ is AC relative to $R_{ix}$ and we could extend $\sigma$ to a solution of $\mathcal{N}$ by assigning any valid value to $v_x$.} Suppose that $E_x$ has $q\ge 2$ constraints and let them be $((v_1,v_x),R_1)$, $\ldots$, $((v_q,v_x),R_q)$. We define a new problem $\mathcal{N}'=\la V,{\D},E_x\ra$ as illustrated in Fig.~\ref{proof3}. 
{Since} $v_x$ is AC relative to all relations in $E_x$, it is easy to verify that $\mathcal{N}'$ has a solution. For example, one can construct a solution of $\mathcal{N}'$ by simply picking a value from $D_x$ for $v_x$ and then extending that valuation to $v_1,\ldots,v_q$. Now, we construct a $q$-ary relation $R$ $=$ $\{\langle \gamma(v_1),\ldots,\gamma(v_q) \rangle\ \mid \gamma $ is a solution of $  \mathcal{N}'\}$. The solution set $S$ of $\mathcal{N}'$ can be obtained by using a sequence of the Cartesian product, equality selection, and projection operations \cite{jeavons1998constraints}. Therefore, $S\in \Gamma^+$. Since   $R=\pi_{v_1,\ldots,v_q}(S)$, we have $R\in\Gamma^+$. By Theorem~\ref{theorem-jeavons},  $R$ should be 2-decomposable\black {; h}owever, \black {in the sequel} we show that it is not, which is a contradiction.

{Let $t= \langle \sigma(v_1),\ldots,\sigma(v_q) \rangle$, where $\sigma$ is a solution of $\mathcal{N}_{-x}$. It is clear that $t$ is a solution of $\mathcal{N}'{\mid_{\{v_1,\ldots,v_q\}}}$. For any list of indices $I$ chosen from $\{1,\ldots,q\}$, with $|I| \le 2$, we claim that $\pi_I(t)\in \pi_I(R)$. We {recall} that, for any two relations $R_{ix},R_{jx}\in E_x$, the relation between $v_i$ and $v_j$ in $\mathcal{N}_{-x}$ is $R_{ij}  \cap (R_{ix}\circ R_{xj})$. Therefore, any partial solution $\langle \sigma(v_i),\sigma(v_j) \rangle (1\le i,j \le q)$ of $\mathcal{N}'$ can be consistently extended to $v_x$ and, by the construction of $\mathcal{N}'$, further consistently extended to a solution of $\mathcal{N}'$. Thus, $\pi_I(t)$ is in $\pi_I(R)$ for any list of indices $I$ chosen from $\{1,\ldots,q\}$, with $|I| \le 2$. However, $t\not\in R$ because $\sigma$ cannot be extended to a solution of $\mathcal{N}'$, which implies that $R$ is not 2-decomposable.} 
~\hfill\qed\end{proof}

\section{The Variable Elimination Algorithm {\dpcplus} }\label{dpcplus}

This section presents a variant of \dpc for solving {BCNs} defined over any weak VEP class. The new algorithm, called \dpcplus and presented as Algorithm~\ref{algorithm_strongppc+}, can solve problems that are not solvable by \dpc (cf. Example~\ref{ex:dpc+} and Proposition~\ref{pro:limitationofdpc}). Compared with the variable elimination algorithm for solving CRC constraints \cite{zhang2009solving}, \dpcplus enforces a weaker AC condition instead of full AC.
We first justify the correctness of Algorithm~\ref{algorithm_strongppc+}. 

\begin{algorithm}[t]  
   \DontPrintSemicolon
   \algsetup{linenosize=\tiny}
   \footnotesize
   \SetAlCapFnt{\footnotesize}
   \SetAlCapNameFnt{\footnotesize}
   \SetAlFnt{\footnotesize}
   \SetKwInOut{Input}{Input}

   \SetKwInOut{Output}{Output}

   \Input{{A binary constraint network $\mathcal{N}=\la V,{\D},C \ra$;} \newline {an} ordering $\prec$ $=$ $(v_1,\ldots,v_n)$ on $V$.}
   \Output{{An equivalent \black {sub}network that is decomposable relative to $\prec$, or ``Inconsistent''.} }
   \BlankLine
   
       \For{$k\gets n$ \KwTo $1$}{

\If{$k$ \emph{has only one successor and that successor is} $i$}{\label{alo2:ln:2}
                   
                       $D_i \gets D_i \cap R_{ki}({D_k})$ \label{ln:AC1}
                       
                       \If{{$D_i= \varnothing$}}{ \label{ln:inconsistent1}
                       
                          \Return{{``Inconsistent''}} \label{alo2:ln:5}
                       
                       }
                   
                   }\Else{
                   
       \For{${i < k}$ \emph{with}  $R_{ik}\in C$}{ \label{alo2:ln:7}
                       $D_k \gets D_k \cap R_{ik}(D_i)$ \label{ln:AC2}
                       
                       \If{{$D_k= \varnothing$}}{\label{ln:inconsistent2}
                       
                          \Return{{``Inconsistent''}} \label{alo2:ln:10}
                       
                       }
                       }
                       \For{${i < k}$ \emph{with}  $R_{ik}\in C$}{ \label{alo2:ln:11}
                    $R_{ik}\gets R_{ik}\cap (D_i \times D_k)$  
                    
  \For{${j < i}$ \emph{with} $R_{jk} \in C$}{
  
  $R_{jk}\gets R_{jk}\cap (D_j \times D_k)$
            
                    \If{$R_{ij} \not\in C$}{
                         $R_{ij}\gets D_i \times D_j$
                         
                         $C \leftarrow C \cup \{R_{ij}\}$
                    }
            
                    $R_{ij} \leftarrow R_{ij} \cap (R_{ik} \circ R_{kj})$; \label{ln:PC}
                
                     \If{$R_{ij} = \varnothing$}{\label{ln:inconsistent3}
                
                          \Return{``Inconsistent''} \label{alo2:ln:20}
              }
              }
              
           }                        
          }
       }
\Return{$\mathcal{N}$}.
\caption{\dpcplus}\label{algorithm_strongppc+}
\end{algorithm}

\begin{theorem} \label{dpc+solveweakvep}
{Let $\Gamma$ be a complete weak VEP class. Suppose $\mathcal{N}$ {is} a BCN defined over $\Gamma$ and {$\prec=(v_1,\ldots,v_n)$} any ordering of variables of $\mathcal{N}$. Then\black {, given $\mathcal{N}$ and $\prec$, Algorithm~\ref{algorithm_strongppc+} does not return ``Inconsistent'' iff $\mathcal{N}$ is consistent.}}
 \end{theorem}
 \begin{proof}
 Suppose the input network $\mathcal{N}$ is consistent. Since \dpcplus only prunes off certain infeasible domain values or relation tuples, the algorithm does not find any empty domains or relations in  Lines~\ref{ln:inconsistent1}, \ref{ln:inconsistent2}, and \ref{ln:inconsistent3}. Thus, it does not return ``Inconsistent''. 

Suppose the algorithm does not return ``Inconsistent'' and let $\mathcal{N}'=\la V,{\D'}$, $C' \ra$ be the output network, where $\D'=\{D'_1, ..., D'_n\}$. We show $\mathcal{N}'$ is consistent. 
 
Write $\mathcal{M}^{(0)}$ for $\mathcal{N}$ and write $\mathcal{M}^{(i)}$ for the result of the $i$-th loop in calling \dpcplus for the input $\mathcal{N}$ and $\prec=(v_1,v_2,...,v_n)$. Then $\mathcal{N}'=\mathcal{M}^{(n-1)}$ and all $\mathcal{M}^{(i)}\ (0\le i < n)$ are equivalent to $\mathcal{N}$. Let $\mathcal{Q}_k$ be the restriction of $\mathcal{M}^{(k)}$ to $\Set{v_1,v_2,...,v_{n-k}}$ ($0\leq k<n$). In essence, $\mathcal{Q}_k$ is obtained by eliminating $v_{n-k+1}$ from $\mathcal{Q}_{k-1}$ (Lines \ref{alo2:ln:2}-\ref{alo2:ln:5} or Lines \ref{alo2:ln:11}-\ref{alo2:ln:20}), while also enforcing AC (Lines \ref{alo2:ln:7}-\ref{alo2:ln:10}) for $v_{n-k+1}$ relative to all its successors if it has more than one successor. Since $\Gamma$ is a complete  weak VEP class, every BCN defined over $\Gamma$ has weak VEP. In particular, each $\mathcal{Q}_{k-1}$ is defined over $\Gamma$ and has weak VEP. This implies that every solution of $\mathcal{Q}_k$ can be extended to a solution of $\mathcal{Q}_{k-1}$. Since no inconsistency is detected in the process, we have $D'_1\not=\varnothing$ and thus $\mathcal{Q}_{n-1}$ is consistent. By the above analysis, this implies that $\mathcal{Q}_{n-2}, ..., \mathcal{Q}_1, \mathcal{Q}_0=\mathcal{M}^{(0)}=\mathcal{N}$ are all consistent. 
 ~\hfill\qed\end{proof}

The preceding proof also gives a way to generate \emph{all}
solutions of a consistent input network \emph{backtrack-free} by
appropriately instantiating the variables along the input ordering $\prec$. Indeed, for all $1 \le k < n$, a
solution $\la a_{1}, \dotsc, a_{k}\ra$ of $\mathcal{N}'_k$ can be
extended to a solution $\la a_{1}, \dotsc, a_{k+1} \ra$ of
$\mathcal{N}'_{k+1}$ by choosing an element $a_{k+1}$ from the
intersection of all $R_{i,k+1}(a_i)$ with
$i \le k$ and $R_{i,k+1} \in C'$, which is always nonempty as shown in the
preceding proof. As we know that if $\Gamma$ is majority-closed, the completion of $\Gamma$ is also majority-closed \cite{jeavons1998constraints}, and that complete majority-closed classes and complete weak VEP classes are equivalent by Theorem~\ref{VEP_SPCP}, this also proves the following result:

\begin{proposition}\label{cor}
Suppose $\mathcal{N}$ is a consistent BCN defined over a majority-closed class and $\prec$ $=$ $(v_1,...,v_n)$ an ordering of variables of $\mathcal{N}$. Then\black {, given
 $\mathcal{N}$ and $\prec$,  Algorithm~\ref{algorithm_strongppc+} returns an
equivalent subnetwork $\mathcal{N}'$ that} is decomposable relative to  $\prec$, \textit{i.e.}, any partial solution of $\mathcal{N}'$ on $\{v_1,...,v_k\}$ {for any $1\leq k <n$} can be extended to a solution of $\mathcal{N}'$.
\end{proposition} 
 
 Note that Lines \ref{alo2:ln:2}-\ref{alo2:ln:10} in \dpcplus  {do} not achieve DAC of input {networks.} Therefore, \dpcplus \emph{does not} achieve strong DPC. Since the overall runtime of Lines \ref{alo2:ln:2}-\ref{alo2:ln:10} is the same {as} enforcing DAC, this {places} \dpcplus in the same time complexity class as \dpc. 
The following example, however, gives a BCN that can be solved by \dpcplus but not by \dpc, which shows that the loop in Lines~\ref{alo2:ln:7}-\ref{alo2:ln:10} is necessary.

\begin{figure}[t!] 
\centering
\includegraphics[width=260pt]{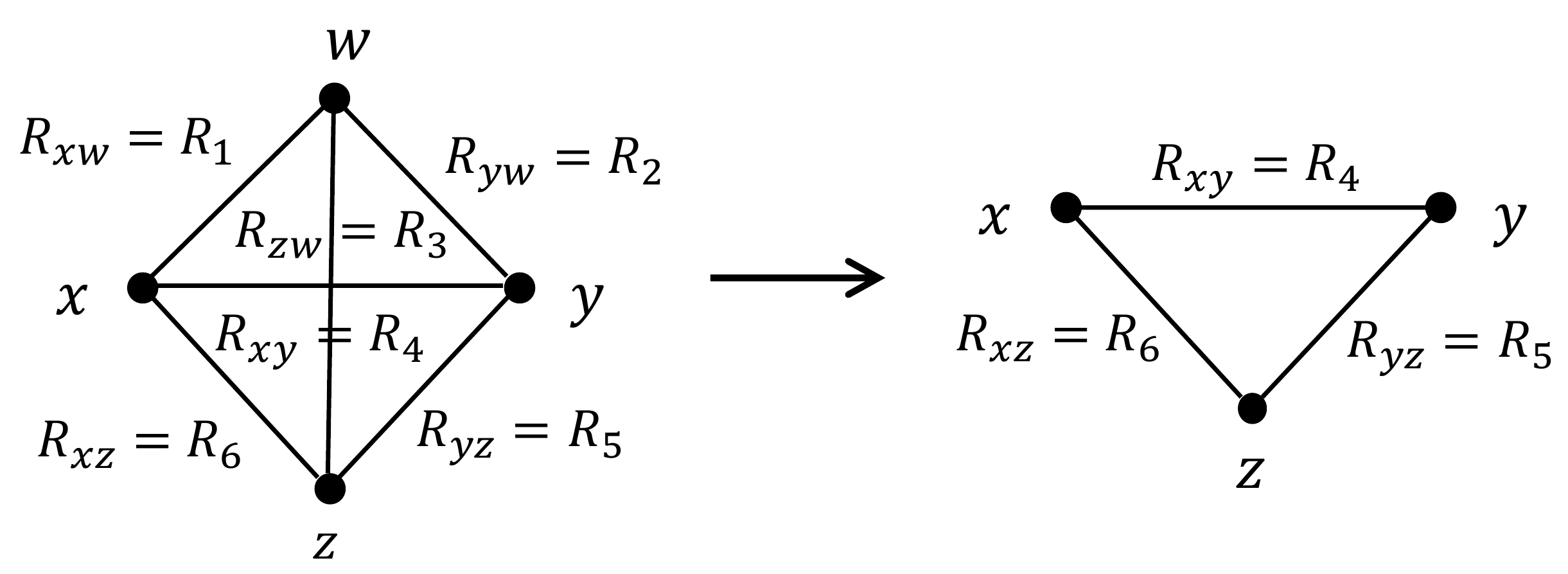}
\caption{{A constraint network $\mathcal{N}$ and its elimination $\mathcal{N}_{-w}$.}}
\label{counter_example1}
\vspace*{-3mm}
\end{figure}

\begin{example}\label{ex:dpc+}
Let $D=\{a,b,c\}$ and $\varphi$ be the majority operation on $D$ such that for all $i,j,k \in D$, $\varphi(i,j,k) = a$ if $i \neq j$, $j \neq k$, and $i \neq k$, and $\varphi(i,j,k) = r$ otherwise, where $r$ is the repeated value (\textit{e.g.}, $\varphi(b,c,b)=b$). Let $\Gamma=\{R_1,R_2,R_3,R_4,R_5,R_6\}$, where $R_1=\{\langle a,a\rangle,\langle a,c\rangle\}$, $R_2=\{\langle c,c\rangle,\langle c,b\rangle\}$, $R_3=\{\langle b,b\rangle,\langle b,a\rangle\}$, $R_4=\{\langle a,c \rangle\}$, $R_5=\{\langle c,b \rangle\}$, and $R_6=\{\langle a,b \rangle\}$. Every $R\in \Gamma$ is closed under the majority operation $\varphi$ on $D$. Now, {consider} the constraint network $\mathcal{N}\in \csp(\Gamma)$ as presented in Figure~\ref{counter_example1}. Since $R_{xw} \circ R_{wz}=R_6$, $R_{xw} \circ R_{wy}=R_4$, and $R_{yw} \circ R_{wz}=R_5$, the eliminated network $\mathcal{N}_{-w}$ is the same as the restriction of $\mathcal{N}$ to the set of variables $\{{v_x,v_y,v_z}\}$. Let $\sigma({v_x})=a,\sigma({v_y})=c,\sigma({v_z})=b$. Then $\sigma$ is a solution of $\mathcal{N}_{-w}$, but $\sigma$ cannot be extended to a solution of {$\mathcal{N}$}. Thus, $\mathcal{N}$    and hence $\Gamma$ do not have VEP. By Theorem~\ref{proposition_VEP_DPCP}, \dpc does not decide $\csp(\Gamma)$. 

On the other hand, since $\Gamma$ is majority-closed, by Proposition~\ref{cor}, \dpcplus can correctly decide the consistency of $\mathcal{N}$. This observation is confirmed by calling the two algorithms on $\mathcal{N}$. Take the PEO $\prec$ $=$ $(w,x,y,z)$ as an example; the other PEOs are  analogous. Let $(\mathcal{N},\prec)$ be an input to \dpc. After processing $w$, we have $D_x=\{a\},D_y=\{c\},D_z=\{b\}$ and $R_{xy}=\{\la a,c \ra\}$, $R_{xz}=\{\la a,b \ra\}$, $R_{zy}=\{ b,c \}$. We can observe that $\la x=a,y=c,z=b\ra$ is a solution to the eliminated subnetwork. Thus, if we keep running \dpc, it will not detect inconsistency. On the other hand,  for \dpcplus, when eliminating $w$, \dpcplus makes $w$ AC relative to its neighbors. Note that \dpc does not perform this operation. After that, $D_w$ is empty, and the algorithm will stop and output ``Inconsistent". 
\end{example}

The following result shows that complete majority-closed classes over domains with more than two elements cannot have VEP.
\begin{proposition}\label{pro:limitationofdpc}

Let $\varphi=\{\varphi_1,\ldots,\varphi_n\}$ be a  majority operation on $\D=\{D_1,\ldots,D_n\}$. If there exists a domain $D_i$ in $\D$ that contains more than two elements, then the set $\Gamma_\varphi$ of binary relations that are closed under $\varphi$ has neither the Helly property nor VEP.  
\end{proposition}
\begin{proof}
Suppose $a,b,c$ are three different values from $D_i$. It is easy to see that the relations $R_1 =\Set{\la a,a\ra,\la a,b\ra}$, $R_2=\Set{\la a, b\ra,\la a,c\ra}$, and $R_3=\Set{\la a,a\ra,\la a,c\ra}$ are all closed under $\varphi$. Therefore, $R_1$, $R_2$, and $R_3$ are all in $\Gamma_{\varphi}$. Because any two of $R_1(a),R_2(a),R_3(a)$ have a common element but $R_1(a)\cap R_2(a)\cap R_3(a)=\varnothing$, this shows that $\Gamma_\varphi$ does not have the Helly Property and, hence by Theorem~\ref{prop:helly+}, does not have VEP.\qed
\end{proof}

This result shows that no complete VEP class could have a domain with 3 or more values. Therefore, there are no interesting complete constraint languages except the boolean ones that can be decided by \dpc (cf. Proposition~\ref{proposition_VEP_DPCP}), while all binary majority-closed classes (including CRC and tree-preserving constraints) can be decided by \dpcplus (cf. Proposition~\ref{cor}). 


\dpcplus can also be used to solve majority-closed constraints of higher arities. This is because, by Theorem~\ref{theorem-jeavons}, every relation definable in a majority-closed language is $2$-decomposable. Therefore, for each majority relation $R$ of arity $m>2$, if a constraint $c=((y_1,...,y_m), R)$ appears in a constraint network $\mathcal{N}$, we could replace $c$ with a set of binary constraints $c_{ij}=((y_i,y_j) \mid \pi_{ij}(R))$ $(1\leq i <j\leq m$), where $\pi_{ij}(R)=\Set{\la t[y_i],t[y_j]\ra \mid t\in R}$.

\section{Evaluations}\label{evaluation}


\begin{figure}[htb]
  \pgfplotsset{width=.48\linewidth,compat=1.9}
  \hfill
\subcaptionbox{Evaluations in the number $n$ of variables. We set $\rho=0.5$, $l=0.3$, $d=100$.\label{fig5:a}}[.45\linewidth]{
\begin{tikzpicture}
\begin{axis}[
    xlabel={Number of variables [$n$]},
    ylabel={CPU time (sec)},
    xmin=20, xmax=120,
    ymin=0, ymax=350,
    xtick={20,40,60,80,100,120},
    ytick={50,100,150,200,250,300,350,400},
    legend pos=north west,
    legend style={font=\tiny},
    ymajorgrids=true,
    xmajorgrids=true,
    grid style=dashed,
]

    \addplot[
    color=blue,
    mark=square,
    ]
    coordinates {
    (20,76.2)(40,158.4)(60,211.3)(80,258.8)(100,289.9)(120,323.1)
    };
      \addlegendentry{$\mathsf{PC2001}$}
      
    \addplot[
    color=black,
    mark=*,
    ]
    coordinates {
    (20,52)(40,85)(60,120)(80,180)(100,210)(120,230)
    };
      \addlegendentry{$\mathsf{SAC3\text{-}SDS}$}

    \addplot[
    color=red,
    mark=triangle,
    ]
    coordinates {
    (20,15.9)(40,20.1)(60,28.1)(80,35.3)(100,46.6)(120,65.3)
    };
    \addlegendentry{$\mathsf{DPC^*}$}

\end{axis}
\end{tikzpicture}}
\hfill
\subcaptionbox{Evaluations in the size $d$ of domains. We set $\rho=0.5$, $l=0.3$, $n=100$.\label{fig5:b}}[.5\linewidth]{
\begin{tikzpicture}
\begin{axis}[
    xlabel={Domain size [$d$]},
    ylabel={CPU time (sec)},
    xmin=50, xmax=300,
    ymin=0, ymax=1000,
    xtick={50,100,150,200,250,300},
    ytick={0,200,400,600,800,1000},
    legend pos=north west,
    legend style={font=\tiny},
    ymajorgrids=true,
    xmajorgrids=true,
    grid style=dashed,
]

\addplot[
    color=blue,
    mark=square,
    ]
    coordinates {
    (50,172)(100,285)(150,420)(200,560)(250,770)(300,890)
    };
    \addlegendentry{$\mathsf{PC2001}$}
    
    \addplot[
    color=black,
    mark=*,
    ]
    coordinates {
    (50,102)(100,185)(150,320)(200,380)(250,470)(300,530)
    };
      \addlegendentry{$\mathsf{SAC3\text{-}SDS}$}

    \addplot[
    color=red,
    mark=triangle,
    ]
    coordinates {
    (50,20)(100,42)(150,65)(200,82)(250,101)(300,115)
    };
      \addlegendentry{$\mathsf{DPC^*}$}

\end{axis}
\end{tikzpicture}}
\subcaptionbox{Evaluations in the density $\rho$ of networks. \\ We set $l=0.3$, $d=100$, $n=100$.\label{fig5:c}}[.5\linewidth]{
\begin{tikzpicture}
\begin{axis}[
    xlabel={Density of networks [$\rho$]},
    ylabel={CPU time (sec)},
    xmin=0.1, xmax=0.5,
    ymin=0, ymax=450,
    xtick={0.1,0.2,0.3,0.4,0.5},
    ytick={0,50,100,150,200,250,300,350,400,450},
    legend pos=north west,
    legend style={font=\tiny},
    ymajorgrids=true,
    xmajorgrids=true,
    grid style=dashed,
]

\addplot[
    color=blue,
    mark=square,
    ]
    coordinates {
    (0.1,265.8)(0.2,270.3)(0.3,278.0)(0.4,280.0)(0.5,287.5)
    };
    \addlegendentry{$\mathsf{PC2001}$}
    
    \addplot[
    color=black,
    mark=*,
    ]
    coordinates {
    (0.1,65.8)(0.2,87.3)(0.3,118.0)(0.4,170.0)(0.5,190.5)
    };
    \addlegendentry{$\mathsf{SAC3\text{-}SDS}$}
    
    \addplot[
    color=red,
    mark=triangle,
    ]
    coordinates {
    (0.1,13.8)(0.2,21.3)(0.3,28.0)(0.4,38.0)(0.5,47.5)
    };
    \addlegendentry{$\mathsf{DPC^*}$}

\end{axis}
\end{tikzpicture}}
\hfill
\subcaptionbox{Evaluations in the looseness $l$ of constraints. We set $\rho=0.5$, $d=100$, $n=100$.\label{fig5:d}}[.48\linewidth]{
\begin{tikzpicture}
\begin{axis}[
    xlabel={Looseness of constraints [$l$]},
    ylabel={CPU time (sec)},
    xmin=0.2, xmax=0.8,
    ymin=0, ymax=500,
    xtick={0.2,0.3,0.4,0.5,0.6,0.7,0.8},
    ytick={0,50,100,150,200,250,300,350,400,450,500},
    legend pos=north west,
    legend style={font=\tiny},
    ymajorgrids=true,
    xmajorgrids=true,
    grid style=dashed,
]

\addplot[
    color=blue,
    mark=square,
    ]
    coordinates {
    (0.2,277.9)(0.3,287)(0.4,291)(0.5,300)(0.6,321.6)(0.7,254.3)(0.8,150.1)
    };
    \addlegendentry{$\mathsf{PC2001}$}
    
\addplot[
    color=black,
    mark=*,
    ]
    coordinates {
    (0.2,107.9)(0.3,148)(0.4,180)(0.5,210)(0.6,230.6)(0.7,184.3)(0.8,107.1)
    };
    \addlegendentry{$\mathsf{SAC3\text{-}SDS}$}    
    
\addplot[
    color=red,
    mark=triangle,
    ]
    coordinates {
    (0.2,30.0)(0.3,33)(0.4,35)(0.5,34)(0.6,34)(0.7,32)(0.8,31)
    };
    \addlegendentry{$\mathsf{DPC^*}$}

\end{axis}
\end{tikzpicture}}

\hfill{}

\caption{Performance comparisons among  $\mathsf{DPC^*}$, $\mathsf{SAC3\text{-}SDS}$ and $\mathsf{PC2001}$ for solving tree-preserving constraint networks. \label{exp_tree}}
\end{figure}

\begin{figure}[htb]
  \pgfplotsset{width=.48\linewidth,compat=1.9}
  \hfill
\subcaptionbox{Evaluations in the number $n$ of variables. We set $\rho=0.5$, $l=0.3$, $d=100$.\label{fig6:a}}[.45\linewidth]{
\begin{tikzpicture}
\begin{axis}[
    xlabel={Number of variables [$n$]},
    ylabel={CPU time (sec)},
    xmin=20, xmax=120,
    ymin=0, ymax=350,
    xtick={20,40,60,80,100,120},
    ytick={50,100,150,200,250,300,350,400},
    legend pos=north west,
    legend style={font=\tiny},
    ymajorgrids=true,
    xmajorgrids=true,
    grid style=dashed,
]

    \addplot[
    color=blue,
    mark=square,
    ]
    coordinates {
    (20,64.8)(40,121.8)(60,165.3)(80,215.8)(100,242.9)(120,269.1)
    };
      \addlegendentry{$\mathsf{PC2001}$}
      
    \addplot[
    color=black,
    mark=*,
    ]
    coordinates {
    (20,43.5)(40,65.4)(60,105)(80,140)(100,173)(120,192)
    };
      \addlegendentry{$\mathsf{SAC3\text{-}SDS}$}

    \addplot[
    color=red,
    mark=triangle,
    ]
    coordinates {
    (20,13.3)(40,16.75)(60,23.6)(80,29.6)(100,38.8)(120,54.3)
    };
    \addlegendentry{$\mathsf{DPC^*}$}

\end{axis}
\end{tikzpicture}}
\hfill
\subcaptionbox{Evaluations in the size $d$ of domains. We set $\rho=0.5$, $l=0.3$, $n=100$.\label{fig6:b}}[.5\linewidth]{
\begin{tikzpicture}
\begin{axis}[
    xlabel={Domain size [$d$]},
    ylabel={CPU time (sec)},
    xmin=50, xmax=300,
    ymin=0, ymax=1000,
    xtick={50,100,150,200,250,300},
    ytick={0,200,400,600,800,1000},
    legend pos=north west,
    legend style={font=\tiny},
    ymajorgrids=true,
    xmajorgrids=true,
    grid style=dashed,
]

\addplot[
    color=blue,
    mark=square,
    ]
    coordinates {
    (50,132)(100,219)(150,353)(200,460)(250,592)(300,704)
    };
    \addlegendentry{$\mathsf{PC2001}$}
    
    \addplot[
    color=black,
    mark=*,
    ]
    coordinates {
    (50,78)(100,142)(150,246)(200,292)(250,361)(300,387)
    };
      \addlegendentry{$\mathsf{SAC3\text{-}SDS}$}

    \addplot[
    color=red,
    mark=triangle,
    ]
    coordinates {
    (50,16)(100,35)(150,54)(200,63)(250,84)(300,100)
    };
      \addlegendentry{$\mathsf{DPC^*}$}

\end{axis}
\end{tikzpicture}}
\subcaptionbox{Evaluations in the density $\rho$ of networks. \\ We set $l=0.3$, $d=100$, $n=100$.\label{fig6:c}}[.50\linewidth]{
\begin{tikzpicture}
\begin{axis}[
    xlabel={Density of networks [$\rho$]},
    ylabel={CPU time (sec)},
    xmin=0.1, xmax=0.5,
    ymin=0, ymax=450,
    xtick={0.1,0.2,0.3,0.4,0.5},
    ytick={0,50,100,150,200,250,300,350,400,450},
    legend pos=north west,
    legend style={font=\tiny},
    ymajorgrids=true,
    xmajorgrids=true,
    grid style=dashed,
]

\addplot[
    color=blue,
    mark=square,
    ]
    coordinates {
    (0.1,224.8)(0.2,225.3)(0.3,231.0)(0.4,233.0)(0.5,240.5)
    };
    \addlegendentry{$\mathsf{PC2001}$}
    
    \addplot[
    color=black,
    mark=*,
    ]
    coordinates {
    (0.1,59.8)(0.2,75.9)(0.3,108.0)(0.4,141.0)(0.5,171.5)
    };
    \addlegendentry{$\mathsf{SAC3\text{-}SDS}$}
    
    \addplot[
    color=red,
    mark=triangle,
    ]
    coordinates {
    (0.1,12.5)(0.2,20.3)(0.3,25.0)(0.4,34.0)(0.5,41.5)
    };
    \addlegendentry{$\mathsf{DPC^*}$}

\end{axis}
\end{tikzpicture}}
\hfill
\subcaptionbox{Evaluations in the looseness $l$ of constraints. We set $\rho=0.5$, $d=100$, $n=100$.\label{fig6:d}}[.48\linewidth]{
\begin{tikzpicture}
\begin{axis}[
    xlabel={Looseness of constraints [$l$]},
    ylabel={CPU time (sec)},
    xmin=0.2, xmax=0.8,
    ymin=0, ymax=500,
    xtick={0.2,0.3,0.4,0.5,0.6,0.7,0.8},
    ytick={0,50,100,150,200,250,300,350,400,450,500},
    legend pos=north west,
    legend style={font=\tiny},
    ymajorgrids=true,
    xmajorgrids=true,
    grid style=dashed,
]

\addplot[
    color=blue,
    mark=square,
    ]
    coordinates {
    (0.2,231.9)(0.3,239)(0.4,271)(0.5,279)(0.6,291.6)(0.7,211.3)(0.8,136.1)
    };
    \addlegendentry{$\mathsf{PC2001}$}
    
\addplot[
    color=black,
    mark=*,
    ]
    coordinates {
    (0.2,93.9)(0.3,120)(0.4,163)(0.5,182)(0.6,201.6)(0.7,167.3)(0.8,90.1)
    };
    \addlegendentry{$\mathsf{SAC3\text{-}SDS}$}    
    
\addplot[
    color=red,
    mark=triangle,
    ]
    coordinates {
    (0.2,29.1)(0.3,31.4)(0.4,32.7)(0.5,34)(0.6,32.7)(0.7,33.6)(0.8,30.1)
    };
    \addlegendentry{$\mathsf{DPC^*}$}

\end{axis}
\end{tikzpicture}}

\hfill{}

\caption{Performance comparisons among  $\mathsf{DPC^*}$, $\mathsf{SAC3\text{-}SDS}$ and $\mathsf{PC2001}$ for solving random majority-closed constraint networks. \label{exp_majority}}
\end{figure}

In this section we experimentally compare algorithm $\mathsf{DPC^*}$ against the state-of-the-art algorithms for solving majority-closed constraint networks. These are $\mathsf{SAC3\text{-}SDS}$ \cite{bessiere2011efficient} and $\mathsf{PC2001}$ \cite{bessiere2005optimal}. $\mathsf{SAC3\text{-}SDS}$ is currently the best \emph{singleton arc-consistency} (SAC) enforcing algorithm~\cite{DBLP:conf/ijcai/DebruyneB97}. Enforcing either SAC\footnote{\emph{Singleton linear arc-consistency} (SLAC) is an alternative  consistency notion that can be enforced to solve majority-closed constraint network \cite{DBLP:conf/lics/Kozik16}, but no practical SLAC algorithms are available so far. 
} or PC correctly decides the consistency of a majority-closed constraint network \cite{chen2011arc,jeavons1998constraints}. \

Two different sets of data are considered for experiments, which are described as follows:


\begin{itemize}
\item[(1)] Tree-preserving constraint networks. 


\item[(2)] Random majority-closed constraint networks. These can be used to test the average performance of different algorithms. To generate such networks, we need to generate random majority-closed constraint languages as follows. 
\begin{itemize}
\item Randomly define a majority operation $\otimes_i: D_i^3 \rightarrow D_i$ for each domain $D_i \in \D$ as follows: for any $ x,y,z\in D_i$, 
\begin{align}
\otimes_i(x,y,z) &= \left\{\begin{array}{cl} \label{eq:1neighbor}
        \text{any } v\in D_i,  &\quad \text{if } x,y,z \text{ are all different,}\\ 
         \text{any repeated value of } x,y,z,  & \quad \text{otherwise.} 
        \end{array}
         \right.
 \end{align}
\item Randomly generate constraints $R_{ij}\subseteq D_i \times D_j$ and test whether
 \begin{equation}\label{eq:maj}
 \{ \langle \otimes_i(t_x[1],t_y[1],t_z[1]),\otimes_j(t_x[2],t_y[2],t_z[2]) \rangle \mid  t_x, t_y, t_z \in R_{ij} \} \subseteq R_{ij}
 \end{equation}
 holds. By definition, $R_{ij}$ is majority-closed under $(\otimes_i,\otimes_j)$ iff \eqref{eq:maj} holds. 

\end{itemize}

\end{itemize}

%

 We used the model in \cite{bliek1999path,DevilleBH99} to generate random consistent constraint networks for experiments. These constraint networks were generated by varying four parameters: (1)~the number of variables $n$, (2)~the size of the domains~$d$, (3)~the density of the constraint networks $\rho$ (i.e. the ratio of non-universal constraints to $n^2$) and (4)~the looseness of constraints $l$ (i.e. the ratio of the number of allowed tuples to $d^2$). We fix three of the four parameters and vary the remaining parameter. Experiments were carried out on a computer with an Intel Core i5-4570 processor with a 3.2 GHz frequency per CPU core, and  4 GB memory.

The graphs in Fig.~\ref{exp_tree} and Fig.~\ref{exp_majority} illustrate the experimental comparisons among algorithms $\mathsf{DPC^*}$, $\mathsf{SAC3\text{-}SDS}$ and $\mathsf{PC2001}$ for solving tree-preserving and random majority-closed constraint networks respectively. The data points in each graph are CPU times averaged over 20 instances.

From Fig.~\ref{exp_tree} and Fig.~\ref{exp_majority}, we observe that all algorithms behave similarly to one another when solving tree-preserving and random majority-closed constraint networks. 
Therefore, our analysis only focuses on Fig.~\ref{exp_tree} and the results are applicable to Fig.~\ref{exp_majority} as well.

 We observe in Fig.~\ref{fig5:a} and ~Fig.~\ref{fig5:b} that all algorithms approximately show linear time behaviors with respect to $n$ and $d$. On the other hand, Fig.~\ref{fig5:c} shows that $\mathsf{PC2001}$ is not sensitive to the density of networks whereas $\mathsf{DPC^*}$ and $\mathsf{SAC3\text{-}SDS}$ perform better when the density of networks is lower. Fig.~\ref{fig5:d} shows that the CPU time for $\mathsf{DPC^*}$ remains almost unchanged when increasing the looseness of constraints. However, the CPU times for  $\mathsf{PC2001}$ and $\mathsf{SAC3\text{-}SDS}$ both go up and then drop down when increasing the looseness of constraints. Finally, we also observe in all the graphs in Fig.~\ref{exp_tree} that the performance differences among $\mathsf{DPC^*}$, $\mathsf{PC2001}$ and $\mathsf{SAC3\text{-}SDS}$ are remarkable. $\mathsf{DPC^*}$ not only runs significantly faster than $\mathsf{PC2001}$ and $\mathsf{SAC3\text{-}SDS}$, but it also scales up to 7 times and 5 times better than $\mathsf{PC2001}$ and $\mathsf{SAC3\text{-}SDS}$ respectively. This mainly owes to the fact that $\mathsf{DPC^*}$ is a single pass algorithm over the ordered input constraint networks.

\section{Conclusion}\label{conclusion}
This paper investigated {which} constraint satisfaction problems can be efficiently decided by enforcing directional path-consistency. Given
a complete binary constraint language $\Gamma$, it turns out that \dpc can decide $\csp(\Gamma)$ if $\Gamma$ is defined over domains with less than three values. For a possibly incomplete binary constraint language $\Gamma$, we proved that $\Gamma$ has the Helly property if, and only if, for any non-trivially inconsistent and triangulated binary constraint network $\mathcal{N}$ over $\Gamma$,  $\mathcal{N}$ is consistent if it is  strongly DPC relative to the inverse ordering of some perfect elimination ordering of the constraint graph of $\mathcal{N}$.  The classes of row-convex \cite{van1995minimality} and tree-convex \cite{ZhangY03} constraints are examples of constraint classes which have the Helly property. More importantly, we presented the algorithm \dpcplus, a simple variant of \dpc, which can decide the CSP of any majority-closed constraint language, and is sufficient for guaranteeing backtrack-free search for majority-closed constraint networks, which have been found applications in various domains, such as scene labeling, temporal reasoning, geometric reasoning, and logical filtering. Our evaluations also show that \dpcplus significantly outperforms the state-of-the-art algorithms for solving majority-closed constraint networks.



\bibliographystyle{splncs03}
\bibliography{reference}

\end{document}